\documentclass[twoside,11pt]{article}

\usepackage{jmlr2e}

\usepackage{amsmath}
\usepackage{algorithm}
\usepackage{algorithmic}
\usepackage{enumitem}

\newtheorem{assump}{Assumption}

\newenvironment{proof-sketch}{\par\noindent{\bf Proof Sketch\ }}{\hfill\BlackBox\\[2mm]}
\newcommand{\RN}[1]{
  \textup{\uppercase\expandafter{\romannumeral#1}}
}

\ShortHeadings{PAC Reinforcement Learning without Real-World Feedback}{Zhong, Deshmukh and Scott}
\firstpageno{1}

\begin{document}

\title{PAC Reinforcement Learning without Real-World Feedback}

\author{\name Yuren Zhong \email zhongyr@umich.edu \\
       \addr Department of EECS\\
       University of Michigan\\
       Ann Arbor, MI 48109, USA
       \AND
       \name Aniket Anand Deshmukh \email andeshm@microsoft.com \\
       \addr Bing Ads \\
        Microsoft AI \& Research \\
        Sunnyvale, CA 94085, USA 
       \AND
       \name Clayton Scott \email clayscot@umich.edu \\
       \addr Department of EECS\\
       University of Michigan\\
       Ann Arbor, MI 48109, USA}

\maketitle
 
\begin{abstract}%   <- trailing '%' for backward compatibility of .sty file

This work studies reinforcement learning in the Sim-to-Real setting, in which an agent is first trained on a number of simulators before being deployed in the real world, with the aim of decreasing the real-world sample complexity requirement. Using a dynamic model known as a rich observation Markov decision process (ROMDP), we formulate a theoretical framework for Sim-to-Real in the situation where feedback in the real world is not available. We establish real-world sample complexity guarantees that are smaller than what is currently known for directly (i.e., without access to simulators) learning a ROMDP with feedback.

\end{abstract}


\section{Introduction}
\label{introduction}

While reinforcement learning has achieved success in many applications, many state-of-art methods require a large number of training samples to find a good policy. In some tasks (e.g., self-driving cars, robotic control), samples are sufficiently costly so as to render existing algorithms infeasible or impractical. One approach to overcoming this challenge, referred to as \textit{Sim-to-Real}, is to first train an agent in one or more simulated environments before deploying it in the real world. Of course, this solution presents its own challenges, owing to the fact that simulators invariably do not coincide with the real-world. Previous works on Sim-to-Real can be classified according to whether the agent does \citep{finn2017model} or does not \citep{tobin2017domain, peng2018sim, bousmalis2018using} receive feedback in the real-world. 

From a theoretical perspective, the goal of Sim-to-Real is to learn a real-world policy that has a smaller sample complexity than if an agent was training only on the real-world. Intuitively, this gain comes from training on the simulators (perhaps with a much larger sample complexity), together with some modelling assumptions that link the simulators and real-world. Existing theoretical studies have focused on the setting where feedback is received in the real world \citep{cutler2015real, jiang2018pac}, and where the dynamics are governed by a conventional Markov decision process (MDP).

Our contribution is to develop a theoretical framework, algorithm, and analysis for Sim-to-Real {\em without} real-world feedback. Furthermore, we study a type of {\em contextual decision process} (CDP) known as a {\em rich observation MDP} (ROMDP) \citep{krishnamurthy2016pac}, which generalizes an MDP by allowing for policies to be based on a (possibly continuous valued) observation associated to an unseen state variable. We establish a real-world sample-complexity guarantee that is smaller than existing guarantees for learning a single ROMDP \citep{krishnamurthy2016pac, jiang2017contextual, sun2019model}. Our modeling assumptions and approach leverage ideas from {\em domain generalization}, reviewed below, which allow us to identify environments based on the marginal distribution of observations.

\section{Related Work}
\label{related work}

There are two theoretical studies, to our knowledge, that formalize Sim-to-Real and give PAC-style bounds. \citet{cutler2015real} assumes a sequence of environments with levels of fidelity from low to high (the highest is the real world and the rest are simulators). There exists a transfer mapping over the states of any two simulators, which is homogeneous. Samples are more expensive in simulators with higher fidelity and switching between simulators is also costly. \citet{jiang2018pac} assumes a single simulator that differs from the real world only at a small fraction of states to get rid of the dependency of the size of action or state space. Both works consider an MDP setting, and both assume feedback is available from the real world.

The idea of a CDP was first introduced by \citet{krishnamurthy2016pac} and extended by \citet{jiang2017contextual}. In general, sample-efficient learning on CDPs with uncountable observation spaces is hard. \citet{jiang2017contextual} provides a sufficient condition for a CDP to be learnable, that is, it admits a Bellman factorization with finite Bellman rank. The ROMDP, a specific type of CDP, was proposed by \citet{krishnamurthy2016pac}, and assumes hidden states with certain dynamics behind the observations. A ROMDP without further assumptions is still difficult to learn. \citet{krishnamurthy2016pac} assume the underlying dynamics are deterministic and the ROMDP is reactive. \citet{dann2018oracle} analyzed computational issues arising in \citet{krishnamurthy2016pac}'s setting via presumed oracles to particular tasks. \citet{jiang2017contextual, sun2019model} studied generalized algorithms for CDPs satisfying some conditions, which partially tighten the sample complexity upper bound in this case. On the other hand, different assumptions are proposed to study ROMDP. \citet{azizzadenesheli2016reinforcement, du2019provably} assume that there is an injective mapping from hidden states to observations, i.e., observations can be partitioned into ``clusters/blocks" where each corresponds to a single state. In our paper, we use a similar setting as \citet{krishnamurthy2016pac}, which is explained in detail in Section \ref{formal setting}.

Sim-to-Real also evokes problems from batch learning that involve generalization to a new task. We highlight two such problems, domain generalization \citep{blanchard2011generalizing} and learning-to-learn or meta-learning \citep{baxter2000model}. In both problems, there are several related labeled training tasks/datasets, and the goal is to generalize to a new task. All tasks are viewed as realizations of a meta-distribution, i.e., a distribution on data-generating distributions. In learning to learn, labeled example are also available for this new task, and the objective is to leverage the training tasks to decrease the sample complexity of learning on the test task, with high probability w.r.t. the draw of the test task. In domain generalization, only unlabeled test data are available, meaning the task must be inferred from the marginal distribution of the input variables. The objective here is to minimize the expected error, where expectation is w.r.t. the draw of the test distribution. This is in contrast to the related problem of {\em domain adapation} where all distributions are considered nonrandom. Like both of these problems, our model for to Sim-to-Real assumes a distribution on learning environments, with simulators and real-world being realizations of a common distribution on environments. Our solution also leverages the idea from domain generalization that the environment can be identified from the marginal distribution on observations. 

%is similar to multi-task learning except for that its goal is better expected performance for a new target task (or testing task) that is unknown beforehand. Domain adaptation \citep{quionero2009dataset, zhang2015multi}, instead, assumes that the unlabelled data of a fixed target task is available throughout the learning process and optimizes the performance on that fixed target task. In learning to learn or meta-learning \citep{baxter2000model, pentina2014pac, maurer2016benefit}, a meta-learner is designed to select an learning algorithm out of a family of candidates that achieve best performance on new tasks. Learning to learn requires labelled data from the new task, which is the major difference from domain generalization and domain adaptation. This paper extends the idea of domain generalization to the reinforcement learning setting to model the popular Sim-to-Real methods.

\section{Problem Formulation}
\label{problem formulation}

Each simulator as well as the real world is modeled as a parameterized ROMDP, which is a contextual decision process with underlying hidden states and an environment parameter. All simulators and the real world share the same state space and the same state-action transitions, but differ in the observation and reward distributions, which are determined by their environment parameters. For any set $W$, we define $\Delta(W)$ to be the set of all distributions on $W$, or the set of all probability density functions on $W$, where the meaning will be clear from the context.

\subsection{Formal Setting}
\label{formal setting}

A deterministic parameterized ROMDP is defined by a tuple $\langle \mathcal{A}, \mathcal{S}, \mathcal{X}, T, s_1, D, R, H, \theta \rangle$, where $H$ is the finite horizon of an episode, $\mathcal{A}, \mathcal{S}$ are the action space and the state space, $\mathcal{X} \subseteq \mathbb{R}^d$ is the observation space that is bounded (i.e. there exists $C_b > 0$ such that $\mathcal{X} \subseteq \mathbb{B}_{\mathbb{R}^d}(0, C_B)$), $T:\mathcal{S}\times\mathcal{A} \to \mathcal{S}$ is the Markov transition dynamic, $s_1 \in \mathcal{S}$ is the fixed initial state, $\theta$ is the environment parameter (a vector that encodes necessary information about the environment), $D_{\theta, s} \in \Delta(\mathcal{X})$ denotes the Lebesgue probability density function over observation space at state $s$, $R_\theta: \mathcal{S}\times\mathcal{X}\times\mathcal{A}\to \Delta([0,1])$ is the parameterized reward distribution. The parameter $\theta$ is drawn from a set of parameters $\Theta$. Without loss of generality, we assume the ROMDP is layered, that is, we can partition $\mathcal{S}$ into $H$ disjoint sets $\mathcal{S}_1, ..., \mathcal{S}_H$ and $\mathcal{X}$ into $H$ disjoint sets $\mathcal{X}_1, ..., \mathcal{X}_H$ such that for any $s_h \in \mathcal{S}_h$ and $a \in \mathcal{A}$, $T(s_h,a) \in \mathcal{S}_{h+1}$, and for any $s_h \in \mathcal{S}_h$ and $\theta \in \Theta$, $D_{\theta, s_h} \in \Delta(\mathcal{X}_h)$. The state space and action space are finite and we define $A = |\mathcal{A}|, S = \max_{h}|\mathcal{S}_h|$. There is no restriction on the observation space (except for the boundedness), which may be countably infinite or continuous. 
% newly added
Only $\mathcal{A}, \mathcal{S}, \mathcal{X}$ and $H$ are assumed to be known beforehand, while $s_1, T, D, R, \theta$ are all hidden from the agent.

There are two aspects of the Markov property in a ROMDP. First, the underlying state transition is Markovian as evidenced by the transition function $T$; second, the observations and the rewards are independent of the previous states and actions given the current state, observation and action as reflected in the definition of $D_{\theta,s}$ and $R_\theta$.

Each episode produces a trajectory $(s_1, x_1, a_1, r_1, ..., s_H, x_H, a_H, r_H)$, where states $s_h = T(s_{h-1}, a_{h-1})$, observations $x_h \sim D_{\theta, s_h}$, rewards $r_h \sim R_\theta(s_h, x_h, a_h)$, and all actions $a_h$ are chosen by some strategy. States and parameters are not observable. For every simulator, an action $a_h$ is made based on the sequence $(x_1, a_1, r_1,...,  x_{h-1}, a_{h-1}, r_{h-1}, x_h)$. However, for the real world, as there is no feedback at all, $a_h$ is made based only on $(x_1, a_1,...,  x_{h-1}, a_{h-1}, x_h)$.

In general, an optimal policy of a ROMDP has to memorize the past trajectory, but in this paper we will assume the ROMDP is reactive (see Section \ref{assumption}). This allows us to restrict our attention to {\em meta-policies}, which are functions $\pi:\Theta\times\mathcal{X}\to\mathcal{A}$ that maps any parameter-observation pair to an action. A meta-policy determines a policy $\pi_\theta = \pi(\theta, \cdot)$ for every environment $\theta$. Then we define the expected total reward $V_\theta(\cdot)$ for a meta-policy $\pi$ in a ROMDP with environment parameter $\theta$ via 
$$
\begin{array}{ccl}
    V_{\theta}(\pi_\theta) & := & V_{\theta}(s_1, \pi_\theta),  \\
    V_{\theta}(s_h, \pi_\theta) & := & \mathbb{E}_{x_h\sim D_{\theta, s_h}}[ r_\theta(s_h, x_h, \pi_\theta(x_h)) + V_\theta(T(s_h, \pi_{{\theta}}(x_h))), \pi_\theta)], \forall h<H, \\
    V_\theta(s_H, \pi_\theta) & := & \mathbb{E}_{x_H\sim D_{\theta, s_H}}[r_\theta(s_H, x_H, \pi_{{\theta}}(x_H))],
\end{array}
$$
where $r_\theta(s, x, a):=\mathbb{E}[R_\theta(s, x, a)]$. Here $V_{\theta}(s_h, \pi_\theta)$ denotes the expected reward of the meta-policy $\pi$ starting from state $s_h$. The optimal meta-policy $\pi^*$ is defined to act optimally for any $\theta$, i.e., $\pi^*$ satisfies
$$V_\theta(\pi^*_\theta) = \sup_{\pi_\theta} V_\theta(\pi_\theta), \forall \theta\in\Theta.$$
The supremum above is taken over all possible policies for environment $\theta$. Since no constraint is imposed on the relation between $\theta$ and $\pi$, such optimal meta-policy $\pi^*$ always exists by taking the combination of the best policies for every $\theta$.

Since in applications there may be different real-world environments (as discussed in Section \ref{introduction}), we formulate the real world as a stochastic environment whose parameter $\theta_R$ is drawn from a prior $\mu\in\Delta(\Theta)$. Moreover, suppose there are $B$ simulators $\beta_1, ..., \beta_B$ associated with $B$ environment parameters $\theta_1, \theta_2,..., \theta_B$, respectively, which are drawn independently according to the same prior $\mu$, i.e. $\theta_R, \theta_1, \theta_2,..., \theta_B \stackrel{iid}{\sim} \mu$. All episodes in the simulator $\beta_b$ / the real world are generated with the corresponding $\theta_b$ / $\theta_R$. We use $\beta$ and $\theta$ interchangeably, like $D_\theta$ and $D_\beta$, $V_\theta$ and $V_\beta$, etc.

% edited by clay
Our goal is to learn a meta-policy that optimizes the expectation of $V_{\theta_R}$ over $\theta_R$. In other words, we want our meta-policy that performs well, in expectation, for any possible real-world environment. An alternative objective would be to perform well for a specific $\theta_R$, but this would require feedback from the real-world, or much stronger assumptions relating the real-world to the simulators.

The expectation of $V_{\theta_R}$, denoted by $V$, is defined as $V(\pi) := \mathbb{E}_{\theta\sim\mu}V_\theta(\pi_\theta).$
Then the definition of $\pi^*$ directly implies that $V(\pi^*) = \sup_{\pi} V(\pi).$ 
Hence, we define the optimal expected total reward $V^* := V(\pi^*).$

Finally we state the problem. A learning algorithm finds a meta-policy $\hat{\pi}_{\theta_R}$ for the real wold by collecting with-feedback trajectories from the $B$ simulators and no-feedback samples from the real world. Our \textbf{goals} are 

i) that with probability at least $1-\delta$, the meta-policy $\hat{\pi}$ is $\epsilon$-optimal, that is, 
$$V^* - \mathbb{E}_{\theta_R \sim \mu}V_{\theta_R}(\hat{\pi}_{\theta_R}) \le \epsilon;$$

ii) to minimize the real-world sample complexity.

\subsection{Main Assumptions}
\label{assumption}

Sample-efficiently learning a near-optimal policy with high probability for a general ROMDP is difficult, and all existing algorithms and PAC bounds are designed for some special cases. \citet{krishnamurthy2016pac} proposed Assumptions 1 and 2 to ensure that a ROMDP is sample-efficiently learnable. 

In reinforcement learning literature, the optimal $Q$-function is usually referred to as the optimal state/observation-action value function, i.e., the expected reward obtained by taking an action at a state/observation and acting optimally afterwards. In this paper we define the optimal $Q$-function as
$$Q^*_{\theta, s_h}(x_h, a_h) := r_\theta(s_h, x_h, a_h) + V^*_\theta(T(s_h, a_h)).$$

A \textit{reactive} policy is a strategy that makes every decision based only on the current observation. For conventional MDPs, there always exists an optimal policy that is reactive. However, for ROMDPs, the Markov property over the action-state transitions does not necessarily imply the Markov property over the action-observation transitions, so making an optimal action requires the memory of the past trajectory. Hence, generally there is no sample-efficient algorithm that can learn a near-optimal policy with high probability from a class of reactive policies \citep[Proposition 1]{krishnamurthy2016pac}. Assumption \ref{A1} has been introduced to guarantee the existence of an optimal reactive optimal policy by assuming that the maximal expected future rewards are independent of the underlying states.

\begin{assump}
(\textit{Partial reactiveness}). For all $a \in \mathcal{A}, x \in \mathcal{X}, \theta \in \Theta$ and any $s, s' \in \mathcal{S}$ such that $D_{\theta, s}(x), D_{\theta, s'}(x) > 0$, $Q^*_{\theta, s}(x, a) = Q^*_{\theta, s'}(x, a)$.
\label{A1}
\end{assump}

Given that the environment parameters are not observable, we will assume that $\theta$ is determined by the marginal distribution of the observations. A predictor $f: \Delta(\mathcal{X})^{\xi} \times \mathcal{X}\times \mathcal{A} \to [0,1]$ ($\xi$ is the total number of the states of the ROMDP, which is no larger than $HS$) takes as input an observation, an action as well as a vector of all marginal distributions of observations at all possible states, and outputs the predicted expected future reward. Generally, predictors should also take state as an argument, but Assumption \ref{A1} ensures the predictors can be independent of the states. Thus, we define the optimal predictor $f^*$, to replace the use of the optimal $Q$-function, as $f^*(D_\theta, x, a) = Q^*_\theta(x, a), \forall \theta, x, a.$

\begin{assump}
(\textit{Realizability}). A class of predictors $\mathcal{F} \subseteq (\Delta(\mathcal{X})^{\xi} \times \mathcal{X}\times \mathcal{A} \to [0,1])$ of size $F = |\mathcal{F}|$ is given, and $f^* \in \mathcal{F}$. 
\label{A2}
\end{assump}

We assume that the class of predictors we choose is a good one, that is, the optimal value function $f^*$ is in this class. We make no assumption on the form of $f^*$ itself. In our main theorem (Theorem \ref{thm:main}), $F$ appears in the sample complexity for simulators but not the real world.

For any $D_\beta \in \Delta(\mathcal{X})^{\xi}$, let $f_{D_\beta} = f(D_\beta, \cdot, \cdot)$. Then a policy $\pi^f_{D_\beta}(x)$ is naturally induced by $D_\beta$ and $f$, that is,
$\pi^f_{D_\beta}(x) = \arg\max_a f_{D_\beta}(x, a).$
We also use $\pi^f_{D}(x)$ to denote the combination of all the policies above for all simulators, i.e., $\pi^f_{D}(x)$ maps $\beta$ to $\pi^f_{D_\beta}$.

\bigskip

The following notation is used to state Assumption \ref{A3}. Suppose $ l = (l_1, ... , l_d)$ is a non-zero vector and $l_1, ... , l_d$ are non-negative integers. Define the 1-norm $|l|$ of $l$ by $|l| := l_1+ \cdot\cdot\cdot +l_d$ and the partial derivative $f^{(l)}$ of any function $f: \mathbb{R}^d \to \mathbb{R}$ by $f^{(l)} := \frac{\partial^{|l|}}{\partial^{l_1} \cdot\cdot\cdot \partial^{l_d}} f.$
\begin{assump}
(\textit{H\"older continuous density}). $D$ is uniformly $\alpha$-H\"older continuous, i.e., there exist $\alpha > 1, C_\alpha$ such that $|D^{(l)}_{\theta, s}(x) - D^{(l)}_{\theta, s}(x')| \le C_\alpha \|x- x'\|^{\alpha-|l|}$ for all $x, x' \in \mathcal{X}, s\in\mathcal{S}, \theta \in \Theta$ and any vector $l$ such that $|l| = \lceil \alpha \rceil - 1$.
\label{A3}
\end{assump}

% edited by clay
Assumption \ref{A3} ensures certain convergence properties of a kernel density estimate of the $D_{\theta,s}$ \citep{Tsybakov2009Introduction, jiang2017uniform}, see Section \ref{kde}. $\alpha$ may be arbitrarily large, if the probability density function is infinitely differentiable.

\begin{assump}
(\textit{Lipschitz continuous value function}). All predictors $f \in \mathcal{F}$ are uniformly Lipschitz continuous over the first parameter, i.e. there exists a constant $C_L$ such that for all $f\in \mathcal{F}, x\in\mathcal{X}, a\in\mathcal{A}$ and any $D_\theta, D'_\theta \in \Delta(\mathcal{X})^{\xi}$, $|f(D_\theta, x, a) - f(D'_\theta, x, a)| \le C_L \cdot \|D_\theta - D'_\theta \|_{\infty}$ where $\|D_\theta - D'_\theta \|_{\infty} := \sup_{x\in\mathcal{X}, s\in\mathcal{S}}|D_{\theta, s}(x) - D'_{\theta, s}(x)|$.
\label{A4}
\end{assump}

Assumption \ref{A4} ensures that small error in density estimation will result in controllable error in future reward prediction. Assumption \ref{A2} and \ref{A4} together imply the true optimal value function is also Lipschitz continuous over the first parameter.

\begin{assump}
(\textit{Distinguishability of states}). There exists a constant $\zeta > 0$ such that for all $\theta \in \Theta$ and any $s, s' \in \mathcal{S}$, $\sup _{x\in\mathcal{X}}|D_{\theta, s}(x) - D_{\theta, s'}(x) | > \zeta$.
\label{A5}
\end{assump}

Assumption \ref{A5} states that given a fixed $\theta$, the difference between the marginal distributions of observations for any two states cannot be arbitrarily small. If we can fully accurately estimate the distributions, say with infinitely many samples, then we are able to distinguish whether any two underlying states are actually the same one.

\section{Methodology}
\label{methodology}

We are inspired by the idea in domain generalization that the marginal distributions of the observations (i.e., the domains) is used to fully characterize different environments, both simulators and the real world. In this paper, we use {\em kernel density estimation} (KDE) to approximate the distribution of observations at each state in each environment. For every simulator, a near-optimal policy (w.r.t. this particular simulator) is learned and all those policies combined with the distribution estimates give a near-optimal meta-policy (in terms of the expected total reward) for the real world. In the whole procedure, real-world samples are collected for the sole purpose of estimating the densities of observation $D_{\theta_R, s}$, which does not need any feedback from real world.

\subsection{Kernel Density Estimation}
\label{kde}

As mentioned previously, our algorithm computes estimates of $D_\beta$ for each simulator $\beta$ and use them as the argument of the predictors. Kernel density estimation (KDE) is a well-studied method to estimate any distribution. With observations $x^{(1)}_\beta, ... ,x^{(n)}_\beta $ drawn from simulator $\beta$ at state $s$ (i.e., $x^{(1)}_\beta, ... ,x^{(n)}_\beta \stackrel{iid}{\sim} D_{\beta, s}$), the KDE $\hat{D}_{\beta,s}$ is
\begin{equation}
\hat{D}_{\beta,s}(x) = \displaystyle\frac{1}{n\cdot h^d}\sum_{i=1}^n \kappa\bigg(\frac{x^{(i)}_\beta - x}{h}\bigg).
\label{eq:kde-formula}
\end{equation}
where $h>0$ is the bandwidth, and $\kappa$ is a kernel. In this paper, $\kappa$ is chosen to satisfy the following conditions, 

\begin{enumerate}[align=left, leftmargin=*, label=\textbf{(K\arabic*)}]
    \item \label{K1} $\int \kappa(t) dt = 1$, $\int \|t\|^\alpha|\kappa(t)|dt < \infty$, and $\int t^s \kappa(t)dt = 0$ for any non-zero vector $s = (s_1, ..., s_d)$ s.t. $s_1, ..., s_d$ are non-negative integers and $|s| \le \lceil \alpha \rceil - 1$;
    \item \label{K2} $\|\kappa\|_2 < \infty$, $\|\kappa\|_\infty < \infty$ and $\mathcal{K}:= \{\kappa(\frac{\cdot - x}{h}), x\in\mathcal{X}\}$ is a uniformly bounded VC-class with dimension $\nu$ and characteristic $\Lambda$ for a fixd $h>0$.
\end{enumerate}

A class of functions $\mathcal{G}$ is a uniformly bounded VC-class with dimension $\nu$ and characteristic $\Lambda$ \citep{talagrand1994sharper, talagrand1996new, gine2001consistency, gine2002rates} if $\|\mathcal{G}\|_\infty := \sup_{g\in\mathcal{G}}\|g\|_{\infty} < \infty$ and there exist positive numbers $ \nu, \Lambda$ such that for all probability measure $\Tilde{P}$ on $\mathbb{R}^d$ and all $\rho \in (0 , \|\mathcal{G}\|_\infty)$, the covering number $N(\mathcal{G}, L_2(\Tilde{P}), \rho)$ satisfies
$$N(\mathcal{G}, L_2(\Tilde{P}), \rho) \le \bigg( \frac{\Lambda \cdot \|\mathcal{G}\|_\infty}{\rho} \bigg)^{\nu},$$
where the covering number is the minimal number of open balls of radius $\rho$ w.r.t. $L_2(\Tilde{P})$ distance and centered within $\mathcal{G}$ that cover $\mathcal{G}$.

To show the existence of such a kernel, we provide one construction based on the orthonormal basis of Legendre polynomials \citep{Tsybakov2009Introduction}. Let $\{\psi_m(\cdot)\}_{m=0}^{\infty}$ be the orthonormal basis of Legendre polynomials in $L_2([-1, 1])$ defined by the formulas
$$\psi_0(t)  = \frac{1}{\sqrt{2}} \text{ and } \psi_m(t) = \sqrt{\frac{2m+1}{2}}\frac{1}{2^m m!}\frac{\partial^m}{\partial t^m}\big[(t^2-1)^m\big], m=1, 2, ...,$$
for $t \in [-1,1]$. Then for any $x = (x_1, ..., x_d) \in \mathcal{X}$ define kernel $\Tilde{\kappa}$ as $\Tilde{\kappa}(x) = \gamma(x_1)...\gamma(x_d)$
where $\gamma: \mathbb{R}\to\mathbb{R}$ is defined as

\begin{equation}
\label{eq:gamma}
\gamma(t) = \sum_{m=0}^{\lceil \alpha \rceil -1} \psi_m(0) \psi_m(t)\mathbf{1}[-1\le t\le 1].
\end{equation}

\begin{proposition}
$\Tilde{\kappa}$ satisfies \ref{K1} and \ref{K2}.
\label{prop:kernel}
\end{proposition}

The next section introduces the training algorithm that learns a meta-policy using the $B$ simulators. Our algorithm is inspired by \citet{krishnamurthy2016pac}.

\subsection{Meta-Policy Learning}
\label{algorithm}

Before starting this subsection, we define the concept of paths. A path, denoted by $p$, is a sequence of at most $H$ actions that the agent takes sequentially from the initial state $s_1$. Moreover, we use $p\circ a$ to denote taking action $a$ after $p$. After taking $p$ / $p\circ a$, it is guaranteed to arrive at one of the states, due to the deterministic dynamics. We use $p$ / $p\circ a$ to denote that state and we call that state the terminal state of path $p$ / $p\circ a$. This allows us to use $s$ and $p$ interchangeably in notation, like $D_s$ and $D_p$. Hence, the empty path $p = \emptyset$ is equivalent to $s_1$, because no action has been taken.

\begin{algorithm} 
\caption{Sim2Real($\mathcal{F}, \epsilon, \delta, \kappa$)}
\label{alg:sim2real} 
\begin{algorithmic}[1]
\STATE Set $\phi = \frac{\epsilon}{500H^2\sqrt{A}}$ and $B = \frac{2}{\phi^2}\log(\frac{256H^2SF\log(4HS/\delta)}{\epsilon \delta})$.
\STATE Sample $B$ simulators $\beta_1, ..., \beta_B$ such that each simulator $\beta_b$ is associated with parameter $\theta_b$ and $\theta_1, ..., \theta_B \stackrel{iid}{\sim} \mu$. Let $\mathcal{B} = \{\beta_1, ..., \beta_B\}$.
\STATE $\hat{D} \gets $ DFS-Distribution($\emptyset, \{\}, \mathcal{B}, \kappa, \epsilon, \delta/4$).
\STATE $\mathcal{F} \gets $ DFS-Learn$(\emptyset, \mathcal{B}, \mathcal{F}, \hat{D}, \phi, \delta/4)$.
\STATE Choose any $f \in \mathcal{F}$. Let $\hat{V}^*$ be a Monte Carlo estimate of $V^f(\emptyset)$ (computed in Line 4).
\STATE $\hat{\pi} \gets$ Learn-on-Simulators($\mathcal{B}, \mathcal{F}, \hat{D}, \hat{V}^*, \epsilon, \delta/4$).
\RETURN $\hat{\pi}$.
\end{algorithmic}
\end{algorithm}

\textbf{Sim2Real} (Algorithm \ref{alg:sim2real}) is designed for learning a near-optimal meta-policy for the real world via only simulators, and takes as input $\epsilon, \delta$, a class of predictors $\mathcal{F}$ and a kernel $\kappa$ satisfying the conditions \ref{K1} and \ref{K2}. Sim2Real first generates $B = \Tilde{\mathcal{O}}\big(\frac{H^4A}{\epsilon^2} \cdot \log \frac{FS}{\delta}\big)$ simulators whose environment parameters are sampled according to $\mu$. DFS-Distribution is invoked to compute the KDE for every state on each simulator. DFS-Learn with its helper functions (Consensus and TD-Eliminate) calculates a close approximation of the maximal future reward $V^*(p)$ starting from the path $p$ at which DFS-Learn is invoked. Since DFS-Learn is invoked at the initial state, an approximation of $V^*$ is computed. Learn-on-Simulators finds a near-optimal meta-policy for the real world using the previous KDEs.

\begin{algorithm} 
\caption{DFS-Distribution($p, \hat{D},\mathcal{B}, \kappa, \epsilon, \delta$)} 
\label{alg:dfs-distribution} 
\begin{algorithmic}[1]
\STATE Set $n_{\text{dist}}$ satisfying (\ref{eq:requisite-ndist}) and $\epsilon_{\text{dist}} = \frac{\zeta}{2} , h = (n_{\text{dist}})^{-\frac{1}{2\alpha+d}}$.
\STATE Collect $n_{\text{dist}}$ observations $x^{(i)}_\beta \sim D_{\beta, p}$ for all $\beta \in \mathcal{B}$.
\STATE Compute $\hat{D}'_{\beta, p}$ for all $\beta \in \mathcal{B}$ by KDE using (\ref{eq:kde-formula}).
\IF{ there exists $p'$ already visited by DFS-Distribution s.t. $\sup_{x\in\mathcal{X}}| \hat{D}'_{\beta, p}(x) -  \hat{D}_{\beta, p'}(x)| \le \epsilon_{\text{dist}}$ for all $\beta \in \mathcal{B}$}
\STATE Append $\hat{D}$ of $p$ and all $p$'s descendants to $\hat{D}$, using $\hat{D}$ of $p'$ and $p'$'s descendants that correspond to $p$ and $p$'s descendants, respectively.
\ELSE
\STATE $\hat{D} \gets \hat{D}.\text{append}(\hat{D}_p)$.
\FOR{$a \in \mathcal{A}$}
\STATE $\hat{D} \gets $ DFS-Distribution($p \circ a, \hat{D}, \mathcal{B}, \kappa, \epsilon, \delta$).
\ENDFOR
\ENDIF
\RETURN $\hat{D}$.
\end{algorithmic}
\end{algorithm}

\textbf{DFS-Distribution} (Algorithm \ref{alg:dfs-distribution}) traverses paths via a DFS (depth-first search). Every time it visits a path $p$ (i.e., a node in the DFS), it compute the KDE for the corresponding state for all environments, where $n_{\text{dist}}$ samples are collected for each simulator. $n_{\text{dist}}$ is chosen satisfying
\begin{equation}
C_L C_{\text{dist}} \cdot (\frac{1}{n_{\text{dist}}})^{\frac{\alpha}{2\alpha+d}}\sqrt{\log(n_{\text{dist}}) + \log\frac{(B+1)HSA}{\delta}} \le \frac{\phi}{2} = \frac{\epsilon}{1000H^2A}.
\label{eq:requisite-ndist}
\end{equation}
$C_{\text{dist}}$ is a constant independent of $\epsilon, \delta, H, S, A$ (see details in Appendix \ref{C}). Since going through the entire search tree of DFS results in the sample complexity depending on $A^H$, we use a simple technique to avoid computing the estimates for the same state multiple times. DFS-Distribution checks whether the current path leads to a state that is visited before, using the computed estimates. If so, all paths with prefix $p$ (i.e., the subtree with the root being $p$) are pruned from the DFS; otherwise, it visits $p\circ a$ for all action $a$ as the DFS continues. DFS-Distribution returns a set of vectors of distributions over the observation space, where each vector corresponds to a simulator.

Every time we say collecting $n$ observations $x^{(i)}\sim D_{\beta,p}$, we execute the path $p$ on simulator $\beta$ to enter its terminal state and collect a single observation, and we repeat this procedure for $n$ times. 

\begin{algorithm} 
\caption{DFS-Learn($p, \mathcal{B}, \mathcal{F}, \hat{D}, \phi, \delta$)} 
\label{alg:dfs-learn} 
\begin{algorithmic}[1]
\STATE Set $\epsilon_{\text{test}} = \big(25(H-|p|-2) + 21\big)\sqrt{A}\phi$.
\FOR{$a \in \mathcal{A}$,}
\IF{ Not Consensus($p\circ a, \mathcal{B}, \mathcal{F}, \hat{D}, \epsilon_{\text{test}}, \phi, \frac{\delta/2}{HSA}$)}
\STATE $\mathcal{F} \gets $ DFS-Learn$(p\circ a, \mathcal{B}, \mathcal{F}, \hat{D}, \phi, \delta)$.
\ENDIF
\ENDFOR
\RETURN TD-Eliminate($p, \mathcal{B}, \mathcal{F}, \hat{D}, \phi, \frac{\delta/2}{HS}$).
\end{algorithmic}
\end{algorithm}

\begin{algorithm}
\caption{Consensus($p, \mathcal{B}, \mathcal{F}, \hat{D}, \epsilon_{\text{test}}, \phi, \delta$)} 
\label{alg:consensus} 
\begin{algorithmic}[1]
\STATE Set $n_{\text{test}} = \frac{2\log(2FB/\delta)}{\phi^2}$. 
\STATE For each $\beta \in \mathcal{B}$, collect $n_{\text{test}}$ observations $x^{(i)}_\beta \sim D_{\beta, p}$ from the simulator $\beta$.
\STATE Compute estimates for each value function, 
$$\Hat{V}_\beta^f(p) = \frac{1}{n_{\text{test}}} \sum_{i = 1}^{n_{\text{test}}}f (\hat{D}_\beta, x^{(i)}_\beta, \pi^f_{\hat{D}_\beta}(x^{(i)}_\beta)), \forall f \in \mathcal{F}, \beta \in \mathcal{B}.$$
\RETURN $\mathbf{1}[|\Hat{V}_\beta^f(p, \pi^f_\beta) - \Hat{V}_\beta^g(p, \pi^g_\beta)| \le \epsilon_{\text{test}}, \forall f, g\in \mathcal{F}, \beta \in \mathcal{B}]$.
\end{algorithmic}
\end{algorithm}

\begin{algorithm}
\caption{TD-Eliminate($p, \mathcal{B}, \mathcal{F}, \hat{D}, \phi, \delta$)} 
\label{alg:td-eliminate} 
\begin{algorithmic}[1]
\STATE Require estimates $\Hat{V}_\beta^f(p \circ a), \forall f \in \mathcal{F}, \beta \in \mathcal{B}, a \in \mathcal{A}$.
\STATE Set $n_{\text{train}} = \frac{2\log(4FB/\delta)}{\phi^2}$.
\STATE For each $\beta \in \mathcal{B}$, collect $n_{\text{train}}$ observations $(x^{(i)}_\beta, a^{(i)}_\beta, r^{(i)}_\beta)$ from the simulator $\beta$, where $x^{(i)}_\beta \sim D_{\beta, p}$, $a^{(i)}_\beta$ is chosen uniformly at random, and $r^{(i)}_\beta \sim R(x^{(i)}_\beta, a^{(i)}_\beta)$.
\STATE Let $Risk(f_{\hat{D}_\beta}) = \frac{1}{n_{\text{train}}}\displaystyle \sum_{i = 1}^{n_{\text{train}}}\Big(f(\hat{D}_\beta, x^{(i)}_\beta, a^{(i)}_\beta) - r^{(i)}_\beta - \Hat{V}_\beta^f(p \circ a^{(i)}_\beta)\Big)^2, \forall f \in \mathcal{F}, \beta \in \mathcal{B}$.
\RETURN $\{f\in\mathcal{F}: Risk(f_{\hat{D}_\beta}) \le \displaystyle \min_{f'\in\mathcal{F}}Risk(f'_{\hat{D}_\beta})+ 2\phi^2 + 8\phi +  \frac{22}{n_{\text{train}}}\log(\frac{2FB}{\delta}),\forall \beta \in \mathcal{B}\}$.
\end{algorithmic}
\end{algorithm}

\textbf{DFS-Learn} (Algorithm \ref{alg:dfs-learn}) traverses paths for the purpose of approximating the maximal future reward $V^*(p)$ starting at the path $p$ at which the current series of recursive calls to DFS-Learn is first invoked and meanwhile eliminating those candidate value functions that does not accurately predict $V^*(p')$ where $p'$ is prefixed by $p$. These two parts are done simultaneously because the latter can be considered as a prerequisite of the former. For any predictor $f \in \mathcal{F}$, let $V^f_\beta(p)$ be its prediction on the maximal future reward starting at $p$.
\textbf{TD-Eliminate} (Algorithm \ref{alg:td-eliminate}) finds good value functions from $\mathcal{F}$ according to their \textit{Bellman risks} (Line 4). Value functions with smaller Bellman risks have more accurate prediction. The true value function $f^*$ is retained (with high probability) and the rest of the survivors are similar to $f^*$ in terms of their prediction accuracy (see Proof of Lemma \ref{lm:2}). Hence all value functions retained by TD-Eliminate approximate $V^*(p)$ closely if TD-Eliminate is invoked at $p$. Notice that in the Bellman risk formula (Line 4) we also need to know the estimate of $V^f_\beta(p\circ a)$ (denoted by $\hat{V}^f_{\beta}(p\circ a)$ in the formula). This is why a DFS is used. Pruning is executed when all functions in the current $\mathcal{F}$ have similar predictions on $V^f_\beta(p\circ a)$, which is examined by \textbf{Consensus} (Algorithm \ref{alg:consensus}). 

\begin{algorithm} 
\caption{Learn-on-Simulators($\mathcal{B}, \mathcal{F}, \hat{D}, \hat{V}^*, \epsilon, \delta$)} 
\label{alg:Learn-on-Simulators} 
\begin{algorithmic}[1]
\STATE Set $\epsilon_{\text{demand}} = \epsilon/2, n_1 = \frac{32\log(6HSB/\delta)}{\epsilon^2}$ and $n_2 = \frac{8\log(3SH/\delta)}{\epsilon B}$.
\WHILE{true}
\STATE Pick any $f \in \mathcal{F}$.
\STATE For each $\beta\in\mathcal{B}$, collect $n_1$ trajectories from the simulator $\beta$, according to policy $\pi^f_{\hat{D}_\beta}$, respectively. Let $v_\beta^{(j)}$ be the total reward of the $j^{th}$ trajectory on the simulator $\beta$.
\STATE $\hat{V}(\pi^f_{\hat{D}}) = \frac{1}{n_1\cdot B}\sum_{\beta\in\mathcal{B}}\sum_{j = 1}^{n_1}v_\beta^{(j)}$.
\IF{$|\hat{V}^* - \hat{V}(\pi^f_{\hat{D}})| \le \epsilon_{\text{demand}}$}
\RETURN $\pi^f$.
\ENDIF
\STATE Update $\mathcal{F}$ by calling DFS-Learn($p, \mathcal{B}, \mathcal{F}, \hat{D}, \phi, \frac{\epsilon\delta}{48 H^2S\log(3HS/\delta)}$) at the terminal states of each of the $H$ prefixes $p$ of each of any $n_2$ paths executed on each simulator in line 4.
\ENDWHILE
\end{algorithmic}
\end{algorithm}

\textbf{Learn-on-Simulators} (Algorithm \ref{alg:Learn-on-Simulators}) loops until a near-optimal meta-policy is found. We use the previously computed estimate $\hat{V}^*$ of the optimal total reward as the reference to determine whether a meta-policy is near-optimal simply via the Monte-Carlo estimation (Line 4). Intuitively, if $f$, which is retained after the execution of DFS-Learn, makes accurate predictions on all states that are visited according to $\pi^f$, then the expected total reward of the induced meta-policy $\pi^f$ is close to $V^*$. However, this is not guaranteed by merely running DFS-Learn because DFS-Learn only ensures that the predictions on the states where DFS-Learn is invoked are accurate. In the case that the current meta-policy $\pi^f$ induced by $f$ is not near-optimal, $\pi^f$ is likely to lead the agent onto unvisited states where $f$ cannot make precise prediction (see Lemma \ref{lm:8}). Hence we invoke DFS-Learn on those unvisited states to further refine $\mathcal{F}$, during which this particular $f$ is supposed to be removed with high probability. This process keeps being iterated till a good meta-policy is found.

Finally, we are able to deploy the learned meta-policy into the real world, combined with the KDE of the real world.

\subsection{Deployment in Real World}
\label{deployment}

In this subsection, we discuss the deployment of the meta-policy $\hat{\pi}$ learned in Section \ref{algorithm}. \textbf{Deploy} (Algorithm \ref{alg:deploy}) collect samples for distribution estimation at paths that are visited by DFS-Distribution previously, and build $\hat{D}_R$ as building $\hat{D}_\beta$. Then we obtain the target policy $\hat{\pi}_{\hat{D}_R} = \hat{\pi}(\hat{D}_R, \cdot)$.

\begin{algorithm} 
\caption{Deploy($\hat{\pi}, \epsilon, \delta$)} 
\label{alg:deploy} 
\begin{algorithmic}[1]
\STATE Let $n_{\text{dist}}, h$ be the same as in DFS-Distribution.
\STATE Collect $n_{\text{dist}}$ observations $x^{(i)}_R$ from the real world at the terminal state of every path visited by DFS-Distribution, and compute $\hat{D}_R$ similarly by KDE.
\RETURN $\hat{\pi}_{\hat{D}_R} = \hat{\pi}(\hat{D}_R, \cdot)$.
\end{algorithmic}
\end{algorithm}

In DFS-Distribution and Deploy, samples are collected for the sole purpose to approximate $D$ so that no feedback is needed. DFS-Learn and Learn-on-Simulators require the collection of rewards, but all are gathered from simulators. Thus, our algorithm does not demand any feedback from the real world.

\section{Result}
\label{result}

Our main theorem is the upper bound of the number of simulators needed and the numbers of samples collected from simulators and the real world, respectively.

\begin{theorem}[Main Theorem]
Suppose Assumption \ref{A1}, \ref{A2}, \ref{A3}, \ref{A4}, \ref{A5} are satisfied and $\epsilon \in (0, 250H^2\sqrt{A}C_L\zeta], \delta \in (0,1)$. Then with probability at least $1-\delta$, Sim2Real and Deploy together find an $\epsilon$-optimal policy $\hat{\pi}_{\hat{D}_R}$, that is, $V^* - \mathbb{E}_{\theta_R \sim \mu}V_{\theta_R}(\hat{\pi}_{\hat{D}_R}) \le \epsilon$. Moreover, at most 
$$\displaystyle \Tilde{\mathcal{O}}\Big(\frac{H^{11}S^2A^3}{\epsilon^5} \cdot (\log F)^2 \cdot (\log\frac{1}{\delta})^3 + \frac{H^4A}{\epsilon^2} \cdot \log \frac{FS}{\delta} \cdot \big(\frac{H^2\sqrt{A}}{\epsilon}\log \frac{S}{\delta}\big)^{2+\frac{d}{\alpha}}\Big)$$
episodes are collected from at most
$$\displaystyle \Tilde{\mathcal{O}}\Big(\frac{H^4A}{\epsilon^2} \cdot \log \frac{FS}{\delta}\Big)$$
simulators and at most
$$\displaystyle \Tilde{\mathcal{O}}\Big( \big(\frac{H^2\sqrt{A}}{\epsilon}\log \frac{S}{\delta}\big)^{2+\frac{d}{\alpha}}\Big)$$
no-feedback episodes are collected from the real world.
\label{thm:main}
\end{theorem}

\begin{proof-sketch}[Theorem \ref{thm:main}]
The proof of Theorem \ref{thm:main} comprises the following major parts:

\noindent
\textbf{(1)} The error of KDE in DFS-Distribution and Deploy is upper bounded by $\frac{\phi}{2C_L}$ simultaneous with probability at least $1-\delta$.

\noindent
\textbf{(2)} Using Assumption \ref{A4}, the error of the prediction by any $f\in\mathcal{F}$, i.e., $|f(D_\beta, x, a) - f(\hat{D}_\beta, x, a)|$, is upper bounded by $\frac{\phi}{2}$.

\noindent
\textbf{(3)} Intuitively, with some probability, Consensus correctly answers whether the current $\mathcal{F}$ have similar predictions on $V^f_\beta(p)$ and TD-Eliminate ensures that $f^*$ is always retained and all remaining functions make prediction with controllable error. Hence, DFS-Learn guarantees that with probability at least $1-\delta$, $f^*$ is retained and any remaining $f, g \in \mathcal{F}$ satisfy $|\hat{V}^f_\beta(s_h) - V^f_\beta(s_h)| \le \phi$ and $\big|V_\beta^f(s_h) - V_\beta^g(s_h)\big|  \le + (H-h+1)(25\sqrt{A}\phi)$, for all $s_h\in\mathcal{S}$ and $\beta \in \mathcal{B}$.

\noindent
\textbf{(4)} If all calls to DFS-Learn are successful, then for any remaining $f \in \mathcal{F}$, the following holds:
$V^* - V(\pi^f_{\hat{D}}) \le 77 H^2\sqrt{A}\phi + \frac{1}{B}\sum_{\beta\in\mathcal{B}}\mathbb{P}(s_1, \pi^f_{\hat{D}_\beta} \to \Bar{L})$ where $\mathbb{P}(s_1, \pi^f_{\hat{D}_\beta} \to \Bar{L})$ is the probability that $\pi^f_{\hat{D}_\beta}$ leads the agent to an unlearned state on the simulator $\beta$. Therefore, a near-optimal meta-policy is guaranteed to be found after all states are visited by DFS-Learn (Line 9 in Algorithm \ref{alg:Learn-on-Simulators}).

\noindent
\textbf{(5)} Combine samples collected in every procedure to give the upper bounds of sample complexity of simulators as well as the real world.
\end{proof-sketch}

The $\Tilde{\mathcal{O}}$ notation hides lower-order logarithmic terms (if $Z$ appears in a bound, $\log Z$ is ignored; if $\log Z$ appears, $\log\log Z$ is ignored, etc.). The simulator sample complexity consists of two parts, one for the KDE and the other for finding a near-optimal meta-policy; on the other hand, the real-world sample complexity only relies on KDE. 

As mentioned in Section \ref{assumption}, if the underlying true $D$ is infinitely differentiable, we can always find a sufficiently large $\alpha$. In this case, we can reduce the exponents of the real-world sample complexity, as $\alpha$ goes to infinity, asymptotically to 
$$\displaystyle \Tilde{\mathcal{O}}\Big( \frac{H^4A}{\epsilon^2}\big(\log \frac{S}{\delta}\big)^2\Big).$$ 

\citet{krishnamurthy2016pac} first studied ROMDP in this setting and \citet{jiang2017contextual} and \citet{du2019provably} improved the upper bound of the sample complexity, with which we compare our result. These methods directly learn from the real-world (without access to simulators), and thus require feedback. The method of \citet{krishnamurthy2016pac} requires at most $\Tilde{\mathcal{O}}\big(\frac{H^6A^2S}{\epsilon^3}\log F (\log\frac{1}{\delta})^2\big)$ samples, while the methods of \citet{jiang2017contextual} and \citet{sun2019model}\footnote{In \cite{sun2019model}, given that it is a model-based method, a set of models is added as part of the input of the algorithm, but we still use $F$ to denote the cardinality of this set of models.} require $\Tilde{\mathcal{O}}\big(\frac{H^5AS^2}{\epsilon^2}\log\frac{F}{\delta}\big)$ samples. It should be kept in mind that unlike these other works, our performance measure is an expected reward averaged over all possible real worlds. Nonetheless, we believe these are the most natural benchmarks for our result.

\section{Conclusion \& Future Work}

Our paper is motivated by Sim-to-Real applications in RL. We are the first to model Sim-to-Real with continuous observations and prove a PAC upper bound on sample complexity, in a domain generalization setting. We propose an algorithm that collects with-feedback samples from simulators to learn a near-optimal meta-policy and is afterwards deployed into real-world environments with collected no-feedback samples from the real world. We prove that the number of simulators is $\text{poly} (H, A, \log S, \log F, 1/\epsilon, \log(1/\delta))$ and the total number of simulator samples is $\text{poly} (H, A, S, \log F, 1/\epsilon, \log(1/\delta))$. More importantly, we prove that the real-world sample complexity (without feedback) is $\Tilde{\mathcal{O}}\big( \big(\frac{H^2\sqrt{A}}{\epsilon}\log \frac{S}{\delta}\big)^{2+\frac{d}{\alpha}}\big)$, which is better than learning directly in the real world with state-of-art algorithms if the underlying distributions of observations are infinitely differentiable.

One interesting direction for future work is to extend from the domain generalization setting to the learning-to-learn setting where there is feedback in the real-world. This may enable even smaller real-world sample complexity, and our policy may provide a useful initial policy in this setting.

% Another potential future direction is whether we can find a framework such that it guarantees the sample-efficiency (and reduce the sample complexity) in a more general setting as long as the CDP itself is sample-efficiently learnable.

%\acks{ }

\vskip 0.2in
\bibliography{references}

\newpage

\appendix

\section{Proof of Proposition \ref{prop:kernel}}
\label{A}

\vskip 0.2in

Recall that in Section \ref{assumption} we have defined the 1-norm $|l| := l_1+ \cdot\cdot\cdot +l_d$ and the partial derivative $f^{(l)} := \frac{\partial^{|l|}}{\partial^{l_1} \cdot\cdot\cdot \partial^{l_d}} f$ for any function $f: \mathbb{R}^d \to \mathbb{R}$ and any non-zero vector $ l = (l_1, \cdot\cdot\cdot, l_d)$ such that $l_1, \cdot\cdot\cdot, l_d$ are non-negative integers. We further define the following notation,
$$l! := l_1 \cdot\cdot\cdot l_d,$$
$$x^l := x_1^{l_1} \cdot\cdot\cdot x_d^{l_d}, \text{ for $x \in \mathcal{X}$}.$$

To prove Proposition \ref{prop:kernel}, we first show Lemma \ref{lm:k1} and \ref{lm:k2}.

\begin{lemma}[\citet{Tsybakov2009Introduction}, Proposition 1.3]
\label{lm:k1}
$\gamma$ defined in (\ref{eq:gamma}) satisfies \
$$\int \gamma(u)du = 1,$$
$$\int u^j\gamma(u)du = 0 \text{ for } j=1,...,\lceil\alpha\rceil-1.$$
\end{lemma}

\begin{lemma}[\cite{nolan1987u, kim2018uniform}]
\label{lm:k2}
If there exists a polynomial function $u$ and a bounded real function $w$ of bounded variance such that the kernel $\kappa$ satisfies $\kappa(x) = w(u(x))$, then $\kappa$ satisfies \ref{K2}
\end{lemma}

\noindent
\textbf{Proposition \ref{prop:kernel}} \textit{$\Tilde{\kappa}$ satisfies \ref{K1} and \ref{K2}.}

\bigskip

\begin{proof}[Proposition \ref{prop:kernel}] 
We prove Proposition \ref{prop:kernel} by two steps. First, we show \ref{K1} is satisfied. Suppose $t = (t_1, ..., t_d)$ is a vector in $\mathbb{R}^d$. Since $\Tilde{\kappa}$ is a polynomial function with bounded support, we know $\int \|t\|^\alpha|\kappa(t)|dt < \infty$. Then we compute $\int \Tilde{\kappa}(t) dt$ and $\int t^s \Tilde{\kappa}(t)dt$ for any non-zero vector $s$ s.t. $s_1, ..., s_d$ are non-negative integers and $|s| \le \lceil \alpha \rceil - 1$. We have
$$
\begin{array}{ccl}
    \int \Tilde{\kappa}(t) dt & = & \int \gamma(t_1)\cdot\cdot\cdot\gamma(t_d) dt_1\cdot\cdot\cdot dt_d \\
     & = & \big(\int \gamma(t_1) dt_1\big)\cdot\cdot\cdot\big(\int \gamma(t_d) dt_d\big) \\
     & \stackrel{(\dag)}{=} & 1.
\end{array}
$$
$(\dag)$ is due to Lemma \ref{lm:k1}.
For any non-zero vector $s = (s_1, \cdot\cdot\cdot, s_d)$ s.t. $|s| \le \lceil \alpha \rceil - 1$, we have
$$
\begin{array}{ccl}
    \int t^s \Tilde{\kappa}(t)dt & = & \int t_1^{s_1}\cdot\cdot\cdot t_d^{s_d} \gamma(t_1)\cdot\cdot\cdot \gamma(t_d) dt_1\cdot\cdot\cdot dt_d \\
    & = & \big(\int t_1^{s_1} \gamma(t_1) dt_1\big)\cdot\cdot\cdot \big(\int t_d^{s_d} \gamma(t_d) dt_d\big) \\
    & \stackrel{(\ddag)}{=} & 0
\end{array}
$$
$(\ddag)$ is due to Lemma \ref{lm:k1} and that at least one of $s_1, ..., s_d$ is positive. Therefore, $\Tilde{\kappa}$ satisfies \ref{K1}. 

Second, $\Tilde{\kappa}$ clearly satisfies the precondition in Lemma \ref{lm:k2} since $\Tilde{\kappa}$ is a bounded real function of bounded variance. Hence $\Tilde{\kappa}$ satisfies \ref{K2}.
\end{proof}

\newpage

\section{Proof of Theorem \ref{thm:main}}
\label{B}

\vskip 0.2in

Recall that $n_{\text{dist}}$ is chosen to satisfy
\begin{equation}
C_L C_{\text{dist}} \cdot (\frac{1}{n_{\text{dist}}})^{\frac{\alpha}{2\alpha+d}}\sqrt{\log(n_{\text{dist}}) + \log\frac{(B+1)HSA}{\delta}} \le \frac{\phi}{2}.
\label{eq:kernel_phi}
\end{equation}

\begin{theorem}[Convergence Rate of KDE]
Suppose $P$ is a probability density function that is $\alpha$-H\"older continuous, i.e. there exist $\alpha > 0, C_\alpha$ such that $|P^{(l)}(x) - P^{(l)}(x')| \le C_\alpha \|x- x'\|^{\alpha-|l|}$ for all $x, x' \in \mathcal{X}$, and any vector $l$ such that $|l| = \lceil \alpha \rceil - 1$. $x^{(1)}, ..., x^{(N)}$ are $N$ samples drawn independently according to $P$. Then the KDE 
$$\hat{P}_h(x) = \displaystyle\frac{1}{N\cdot h^d}\sum_{i=1}^N \kappa(\frac{x^{(i)} - x}{h}) $$
of $P$ with a kernel $\kappa$ satisfying \ref{K1} and \ref{K2} has the following convergence rate, that is, with probability at least $1-\delta$, for all $h > 0$, 
$$\sup_{x}|\hat{P}_h(x) - P(x)| \le  C_1 \sqrt{\frac{\log(1/h^d) + \log(1/\delta)}{h^d N}} + C_2 h^\alpha,$$
where $C_1$ depends on $\nu, \Lambda, \|\kappa\|_\infty, \|\kappa\|_2, \|P\|_\infty$ and $C_2$ depends on $C_\alpha, \alpha, \int \big| \kappa(z) \big| \| z \|^{\alpha } dz$.
\label{thm:kde}
\end{theorem}

\textbf{Remark 1.} The optimal convergence rate obtained by choosing $h \approx N^{-\frac{1}{2\alpha+d}}$ is 
$$\sup_{x}|\hat{P}(x) - P(x)| \le \Tilde{\mathcal{O}}( N^{-\frac{\alpha}{2\alpha+d}}),$$ 
which matches the known lower bound \citep{Tsybakov2009Introduction}.

\begin{corollary}[DFS-Distribution \& Deploy]
Suppose DFS-Distribution and Deploy are invoked with the bandwidth $h = (n_{\text{dist}})^{-\frac{1}{2\alpha+d}}$. If $\zeta \ge \frac{2\phi}{C_L},$ i.e., $\epsilon \le 250H^2\sqrt{A}C_L\zeta$, 
then with probability at least $1-\delta$,
\begin{equation}
\|\hat{D}_{\beta} - D_{\beta}\|_\infty \le \frac{\phi}{2C_L}
\label{eq:dfs-dist}
\end{equation}
for all $\beta \in \mathcal{B}$ and 
\begin{equation}
\|\hat{D}_{R} - D_{R}\|_\infty \le \frac{\phi}{2C_L}
\label{eq:dfs-dist-1}
\end{equation}
where $C_{\text{dist}}$ is a constant depending on $d, \alpha, \nu, \Lambda, \|\kappa\|_\infty, \|\kappa\|_2, \|P\|_\infty, C_\alpha, \alpha, \int \big| \kappa(z) \big| \| z \|^{\alpha } dz$. Moreover, DFS-Distribution is called at most $HSA$ times, and the total number of episodes collected from simulators is at most $n_{\text{dist}} BHSA$ while the number of episodes collected from the real world is at most $n_{\text{dist}} HSA$.
\label{crl:dfs-dist}
\end{corollary}

Before further analysis, we need additional notation as follows. For any simulator $\beta$ associated with parameter $\theta$, define the following quantities,
$$V^f_\beta(p, \pi^g) = V^f_\theta(p, \pi^g) =\mathbb{E}_{x\sim D_{\beta, p}}f(D_{\beta}, x, \pi^g_{D_{\beta}}(x)),$$
$$V^f(p, \pi^g) = \mathbb{E}_{\beta\sim \mu}V^f_\beta(p, \pi^g)= \mathbb{E}_{\beta\sim \mu}\mathbb{E}_{x\sim D_{\beta, p}}f(D_{\beta}, x, \pi^g_{D_{\beta}}(x)).$$
For the sake of convenience, let 
$$V^f_\beta(p) = V^f_\theta(p) = V^f_\beta(p, \pi^f) = V^f_\theta(p, \pi^f),$$
$$V^f(p) = V^f(p, \pi^f).$$
In particular, we have $V^{f^*}(p) = V^*(p)$ and $V^{f^*}(\emptyset) = V^*$.

Above we assume that the true distributions $D_\beta$ are inputted to $\pi$, but in the algorithm we can only input the estimates of distributions, $\hat{D}_\beta$. Hence, we then define $V$ for estimates of distributions by
$$V^f_\beta(p, \pi^g_{\hat{D}_\beta}) = V^f_\theta(p, \pi^g_{\hat{D}_\beta}) =\mathbb{E}_{x\sim D_{\beta, p}}f(D_{\beta}, x, \pi^g_{\hat{D}_\beta}(x)),$$
$$V^f(p, \pi^g_{\hat{D}}) = \mathbb{E}_{\beta\sim \mu}V^f_\beta(p, \pi^g_{\hat{D}_\beta})= \mathbb{E}_{\beta\sim \mu}\mathbb{E}_{x\sim D_{\beta, p}}f(D_{\beta}, x, \pi^g_{\hat{D}_\beta}(x)).$$

Moreover, we use the notation $\hat{V}$ to denote the empirical estimate of $V$ as in Consensus. Suppose $n$ samples $x^{(1)}_\beta, ..., x^{(n)}_\beta \sim D_{\beta, p}$ are collected. Note that $n$ is replaced with $n_{\text{test}}$ in Consensus and $n_{\text{train}}$ in TD-Eliminate. Define $\hat{V}^f_\beta(p)$ by
$$\hat{V}^f_\beta(p) = \frac{1}{n} \sum_{i=1}^{n}f(\hat{D}_{\beta}, x^{(i)}_\beta, \pi^f_{\hat{D}_{\beta}}(x^{(i)}_\beta)),$$
and define $\Tilde{V}^f_\beta(p)$ by
$$\Tilde{V}^f_\beta(p) = \frac{1}{n} \sum_{i=1}^{n}f(D_{\beta}, x^{(i)}_\beta, \pi^f_{D_{\beta}}(x^{(i)}_\beta)).$$ 
We also define the variables above for the real world similarly with the subscript replaced by $R$.

\begin{theorem}[Consensus]
Suppose the call to DFS-Distribution is successful, that is, (\ref{eq:dfs-dist}) and (\ref{eq:dfs-dist-1}) hold, and Consensus is invoked on path $p$ with $n_{\text{test}} = \frac{2\log(2FB/\delta)}{\phi^2}, \epsilon_{\text{test}} = \tau_1 + 2\phi$, for some $\tau_1 > 0$. Let $\Tilde{V}^f_\beta(p) = \frac{1}{n_{\text{test}}} \displaystyle \sum_{i=1}^{n_{\text{test}}}f(D_{\beta}, x^{(i)}_\beta, \pi^f_{D_{\beta}}(x^{(i)}_\beta))$. Then with probability at least $1 - \delta$, the following statements hold true simultaneously:

1. For all $f\in\mathcal{F}, \beta \in \mathcal{B}$, $|\Tilde{V}^f_\beta(p) - V^f_\beta(p)| \le \frac{\phi}{2}, |\hat{V}^f_\beta(p) - V^f_\beta(p)| \le \phi$;

2. If $|V^f_\beta(p) - V^g_\beta(p)|\le \tau_1, \forall f, g \in \mathcal{F}, \beta \in \mathcal{B}$, then Consensus returns \textbf{True};

3. If Consensus returns \textbf{True}, then $|V^f_\beta(p) - V^g_\beta(p)| \le \epsilon_{\text{test}} + 2\phi, \forall f, g\in \mathcal{F}, \beta \in \mathcal{B}$.
\label{thm:consensus}
\end{theorem}

\begin{theorem}[TD-Eliminate]
Suppose the call to DFS-Distribution is successful, that is, (\ref{eq:dfs-dist}) and (\ref{eq:dfs-dist-1}) hold, and TD-Eliminate is invoked at path $p$ with $\mathcal{F}, \phi, \delta$ and $n_{\text{train}} = \frac{2\log(4FB/\delta)}{\phi^2}$, and the following conditions are satisfied:

(Precondition 1): $|\Tilde{V}^f_\beta(p\circ a) - V^f_\beta(p\circ a)| \le \frac{\phi}{2}, \forall f \in \mathcal{F}, a\in \mathcal{A}, \beta\in \mathcal{B};$

(Precondition 2): $|\hat{V}^f_\beta(p\circ a) - V^f_\beta(p\circ a)| \le \phi, \forall f \in \mathcal{F}, a\in \mathcal{A}, \beta\in \mathcal{B};$

(Precondition 3): $|V^f_\beta(p\circ a) - V^g_\beta(p\circ a)| \le \tau_2, \forall f,g \in \mathcal{F}, a\in \mathcal{A}, \beta\in \mathcal{B}.$

\noindent
Let $\hat{V}_\beta^f(p) = \frac{1}{n_{\text{train}}}\displaystyle\sum_{i=1}^{n_{\text{train}}}f(\hat{D}_\beta, x_\beta^{(i)}, \pi_{\hat{D}_\beta}^f(x_\beta^{(i)}))$ and $\Tilde{V}_\beta^f(p) = \frac{1}{n_{\text{train}}}\displaystyle\sum_{i=1}^{n_{\text{train}}}f({D}_\beta, x_\beta^{(i)}, \pi_{{D}_\beta}^f(x_\beta^{(i)}))$. Then with probability at least $1-\delta$, the following hold simultaneously:

1. $f^*$ is retained;

2. For all $f\in\mathcal{F}$ retained and for all $\beta\in\mathcal{B}$, $|\Tilde{V}^f_\beta(p) - V^f_\beta(p)| \le \frac{\phi}{2}, |\hat{V}^f_\beta(p) - V^f_\beta(p)| \le \phi$;

3. For all $f,g\in\mathcal{F}$ retained and for all $\beta\in\mathcal{B}$, $\big|V_\beta^f(p) - V_\beta^g(p)\big|  \le   25\sqrt{A}\phi + \tau_2$;

4. For all $f\in\mathcal{F}$ retained and for all $\beta\in\mathcal{B}$, $V^{f^*}_\beta(p) - V^{f^*}_\beta(p, \pi^f_{\hat{D}_\beta}) \le 25\sqrt{A}\phi + 2\tau_2$.
\label{thm:td-eliminate}
\end{theorem}

\begin{theorem}[DFS-Learn]
Suppose the call to DFS-Distribution is successful, that is, (\ref{eq:dfs-dist}) and (\ref{eq:dfs-dist-1}) hold, and DFS-Learn is invoked at path $p$ with $\mathcal{F},\delta,\phi$. With probability at least $1-\delta$, for any $h=1,2,...,H$ and any $s_h\in\mathcal{S}_h$ such that TD-Eliminate is called, the conclusions of Theorem \ref{thm:td-eliminate} hold with $\tau_2 = (H-h)(25\sqrt{A}\phi)$. Moreover, the number of episodes executed on similators by DFS-Learn is at most
$$\mathcal{O}(\frac{HSAB}{\phi^2}\log(\frac{HSAFB}{\delta})).$$
\label{thm:dfs-learn}
\end{theorem}

For the next part of analysis, define $\Bar{V}$ by
$$\Bar{V}(\pi^f_{\hat{D}}) := \frac{1}{B}\sum_{\beta\in\mathcal{B}}V_\beta(\pi^f_{\hat{D}_\beta}),$$
$$\Bar{V}^f(p):= \frac{1}{B}\sum_{\beta\in\mathcal{B}}V_\beta^f(p).$$

\begin{theorem}[Simulators]
Suppose the call to DFS-Distribution is successful, that is, (\ref{eq:dfs-dist}) and (\ref{eq:dfs-dist-1}) hold, and $B = \frac{\log(4F/\delta')}{2\phi^2}$. Then with probability at least $1-\delta'$, for all $f \in \mathcal{F}$,
$$|\Bar{V}(\pi^f_{\hat{D}}) - V(\pi^f_{\hat{D}})| \le \phi,$$
$$|\Bar{V}^f(\emptyset) - V^f(\emptyset)| \le \phi.$$
\label{thm:boundK}
\end{theorem}

\begin{corollary}[Estimate of $V^*$]
Suppose the call to DFS-Distribution is successful, that is, (\ref{eq:dfs-dist}) and (\ref{eq:dfs-dist-1}) hold. Then with probability at least $1-\delta-\delta'$, $\hat{V}^*$ in Line 6 in Sim2Real satisfies
$$|\hat{V}^* - V^*| \le 33H\sqrt{A}\phi.$$
\label{crl:1}
\end{corollary}

\begin{theorem}[Learn-on-Simulators]
Suppose the call to DFS-Distribution is successful, that is, (\ref{eq:dfs-dist}) and (\ref{eq:dfs-dist-1}) hold, Learn-on-Simulators is invoked with $\mathcal{F}, \hat{V}^*, \epsilon, \delta$, and $B = \frac{2}{\phi^2}\log(\frac{192H^2SF\log(3HS/\delta)}{\epsilon \delta})$. Then with probability at least $1 - \delta$, Learn-on-Simulators terminates with outputting a meta-policy $\hat{\pi}$ such that $V^* - V(\hat{\pi}) \le \epsilon$. Moreover, the number of episodes executed on simulators by Learn-on-Simulators is at most
$$\displaystyle \Tilde{\mathcal{O}}\Big(\frac{H^{11}S^2A^3}{\epsilon^5} \cdot (\log F)^2 \cdot (\log\frac{1}{\delta})^3\Big).$$
\label{thm:Learn-on-Simulators}
\end{theorem}

\noindent
\textbf{Theorem \ref{thm:main}}
\textit{Suppose Assumption \ref{A1}, \ref{A2}, \ref{A3}, \ref{A4}, \ref{A5} are satisfied and $\epsilon \in (0, 250H^2\sqrt{A}C_L\zeta], \delta \in (0,1)$. Then with probability at least $1-\delta$, Sim2Real and Deploy together find an $\epsilon$-optimal policy $\hat{\pi}_{\hat{D}_R}$, that is, $V^* - \mathbb{E}_{\theta_R \sim \mu}V_{\theta_R}(\hat{\pi}_{\hat{D}_R}) \le \epsilon$. Moreover, at most 
$$\displaystyle \Tilde{\mathcal{O}}\Big(\frac{H^{11}S^2A^3}{\epsilon^5} \cdot (\log F)^2 \cdot (\log\frac{1}{\delta})^3 + \frac{H^4A}{\epsilon^2} \cdot \log \frac{FS}{\delta} \cdot \big(\frac{H^2\sqrt{A}}{\epsilon}\log \frac{S}{\delta}\big)^{2+\frac{d}{\alpha}}\Big)$$
episodes are collected from at most
$$\displaystyle \Tilde{\mathcal{O}}\Big(\frac{H^4A}{\epsilon^2} \cdot \log \frac{FS}{\delta}\Big)$$
simulators and at most
$$\displaystyle \Tilde{\mathcal{O}}\Big( \big(\frac{H^2\sqrt{A}}{\epsilon}\log \frac{S}{\delta}\big)^{2+\frac{d}{\alpha}}\Big)$$
episodes are collected from the real world.}

\bigskip

\begin{proof}[Theorem \ref{thm:main}]
In Sim2Real, we assign $\delta/4$ and $3\delta/4$ to the parameters $\delta$ in Corollary \ref{crl:dfs-dist} and Theorem \ref{thm:Learn-on-Simulators}, respectively, as in the algorithm. Then by combining Corollary \ref{crl:dfs-dist}, Theorem \ref{thm:dfs-learn}, Corollary \ref{crl:1} and Theorem \ref{thm:Learn-on-Simulators}, we complete the proof.
\end{proof}

\newpage

\section{Proof of Theorem \ref{thm:kde} \& Corollary \ref{crl:dfs-dist}}
\label{C}

\vskip 0.2in

The error of KDE consists of two parts, the bias error and the estimation error. Let $P_h(x) := \mathbb{E}[\hat{P}_h(x)] = \frac{1}{h^d} \int \kappa(\frac{u - x}{h})P(u)du$. With a fixed $h$, the KDE $\hat{P}_h$ ultimately converges to $P_h$, which is different from the true probability $P$. The estimation error refers to the $\|\hat{P}_h - P_h\|$ while the bias error is the $\|P_h - P\|$. We use Lemma \ref{lm:kde} and Lemma \ref{lm:kde_sq} to bound the estimation error and Lemma \ref{lm:kde_bias} to bound the bias error.

\begin{lemma}[\citet{kim2018uniform}, Theorem 30]
Let $P$ be a probability distribution on $\mathbb{R}^d$ and let $x^{(1)}, ..., x^{(N)}$ be i.i.d. from $P$. Let $\mathcal{G}$ be a class of functions from $\mathbb{R}^d$ to $\mathbb{R}$ that is uniformly bounded VC-class with dimension $\nu$, i.e. $\|\mathcal{G}\|_\infty := \sup_{g\in\mathcal{G}}\|g\|_{\infty} < \infty$ and there exist positive numbers $ \nu, \Lambda$ such that for all probability measure $\Tilde{P}$ on $\mathbb{R}^d$ and all $\rho \in (0 , \|\mathcal{G}\|_\infty)$, the covering number $N(\mathcal{G}, L_2(\Tilde{P}), \rho)$ satisfies
$$N(\mathcal{G}, L_2(\Tilde{P}), \rho) \le \bigg( \frac{\Lambda \cdot \|\mathcal{G}\|_\infty}{\rho} \bigg)^{\nu}.$$
Let $\sigma > 0$ with $\mathbb{E}_{{P}}[g^2] \le \sigma^2$ for all $g \in \mathcal{G}$. Then there exists a universal constant $C$ not depending on any parameters such that with probability at least $1-\delta$, 
$$\begin{array}{cl}
     & \displaystyle\sup_{g\in\mathcal{G}}\bigg|\frac{1}{N}\sum_{i=1}^N g(x^{(i)}) - \mathbb{E}_{x\sim P}[g(x)]\bigg| \\
    \le & \displaystyle C \bigg( \frac{\nu \|\mathcal{G}\|_\infty}{N}\log\big(\frac{2\Lambda\|\mathcal{G}\|_\infty}{\sigma}\big) + \sqrt{\frac{\nu \sigma^2}{N}\log\big(\frac{2\Lambda\|\mathcal{G}\|_\infty}{\sigma}\big)} + \sqrt{\frac{\sigma^2 \log(\frac{1}{\delta})}{N}} + \frac{\|\mathcal{G}\|_\infty \log (\frac{1}{\delta})}{N}\bigg).
\end{array}
$$
\label{lm:kde}
\end{lemma}

\begin{lemma}[\citet{sriperumbudur2012consistency}, Proof of Proposition A.5]
For any $x\in\mathcal{X}$ and $h > 0$,
$$\mathbb{E}_{P}\big[\big| \kappa(\frac{\cdot - x}{h})\big|^2\big] \le \|\kappa\|_2 \|P\|_\infty h^d.$$
\label{lm:kde_sq}
\end{lemma}

\begin{lemma}
$P$ defined in Theorem \ref{thm:kde} satisfies that for any $h > 0$,
$$\displaystyle\sup_{x\in\mathcal{X}}\big| P_h(x) - P(x) \big| \le  h^{\alpha} C_\alpha \sum_{|l| = \lceil \alpha\rceil - 1} \frac{1}{l!}  \int \big| \kappa(z) \big| \| z \|^{\alpha } dz$$
\label{lm:kde_bias}
\end{lemma}

\begin{proof}[Lemma \ref{lm:kde_bias}]
For any $h > 0$ and any $x\in\mathcal{X}$,
$$
\begin{array}{ccl}
    \displaystyle  P_h(x) - P(x) & = & \displaystyle \frac{1}{h^d}\int \kappa(\frac{u - x}{h})P(u)du - P(x) \\
    & \stackrel{(\dag)}{=} & \displaystyle\int \kappa(z)\Big[P(x + hz)- P(x)\Big]dz \\
    & \stackrel{(\ddag)}{=} & \displaystyle\int \kappa(z) \Big[ \sum_{m=1}^{\lceil \alpha\rceil - 2} \sum_{|s| = m}\frac{1}{s!} P^{(s)}(x) h^{|s|}z^s + \sum_{|l| = \lceil \alpha\rceil - 1} \frac{1}{l!} P^{(l)}(x+\lambda hz) h^{|l|}z^l \Big] dz \\
    & \stackrel{(\perp)}{=} & \displaystyle\int \kappa(z) \Big[\sum_{|l| = \lceil \alpha\rceil - 1} \frac{1}{l!} \Big( P^{(l)}(x+\lambda hz) - P^{(l)}(x) \Big) h^{|l|}z^l \Big] dz
\end{array}
$$
where $(\dag)$ is obtained by the substitution of variable $z = \frac{u-x}{h}$ and $\int \kappa(t) dt = 1$ in \ref{K1}; $(\ddag)$ by the Taylor expansion with Lagrange remainder ($\lambda \in (0,1)$); $(\perp)$ by $\int t^s \kappa(t)dt = 0$ for any non-zero vector $s$ s.t. $|s| \le \lceil \alpha \rceil - 1$ in \ref{K1}.

Due to the precondition that $P$ is $\alpha$-Hölder continuous, i.e. there exist $\alpha > 0, C_\alpha$ such that $|P^{(l)}(x) - P^{(l)}(x')| \le C_\alpha \|x- x'\|^{\alpha-|l|}$ for all $x, x' \in \mathcal{X}$, and any vector $l$ such that $|l| = \lceil \alpha \rceil - 1$, we can upper bound $| P_h(x) - P(x) |$ by 
$$\begin{array}{ccl}
    \displaystyle \big| P_h(x) - P(x) \big| & = & \displaystyle \bigg|\int \kappa(z) \Big[\sum_{|l| = \lceil \alpha\rceil - 1} \frac{1}{l!} \Big( P^{(l)}(x+\lambda hz) - P^{(l)}(x) \Big) h^{|l|}z^l \Big] dz \bigg| \\
    & \le & \displaystyle h^{\alpha}C_\alpha \lambda^{\alpha - |l|} \sum_{|l| = \lceil \alpha\rceil - 1} \frac{1}{l!} \int \big| \kappa(z) \big|   \| z \|^{\alpha - |l|} \big| z^l \big| dz \\
    & \le & \displaystyle h^{\alpha} C_\alpha \sum_{|l| = \lceil \alpha\rceil - 1} \frac{1}{l!}  \int \big| \kappa(z) \big| \| z \|^{\alpha } dz
\end{array}
$$ 
where the integral in the last line is guaranteed to be finite by \ref{K1}. Since the inequality holds for all $h > 0$ and all $x\in\mathcal{X}$, we complete the proof.
\end{proof}

\noindent
\textbf{Theorem \ref{thm:kde}} \textit{Suppose $P$ is a probability density function that is $\alpha$-H\"older continuous, i.e. there exist $\alpha > 0, C_\alpha$ such that $|P^{(l)}(x) - P^{(l)}(x')| \le C_\alpha \|x- x'\|^{\alpha-|l|}$ for all $x, x' \in \mathcal{X}$, and any vector $l$ such that $|l| = \lceil \alpha \rceil - 1$. $x^{(1)}, ..., x^{(N)}$ are $N$ samples drawn independently according to $P$. Then the KDE 
$$\hat{P}_h(x) = \displaystyle\frac{1}{N\cdot h^d}\sum_{i=1}^N \kappa(\frac{x^{(i)} - x}{h}) $$
of $P$ with a kernel $\kappa$ satisfying \ref{K1} and \ref{K2} has the following convergence rate, that is, with probability at least $1-\delta$, for all $h > 0$, 
$$\sup_{x}|\hat{P}_h(x) - P(x)| \le  C_1 \sqrt{\frac{\log(1/h^d) + \log(1/\delta)}{h^d N}} + C_2 h^\alpha,$$
where $C_1$ depends on $\nu, \Lambda, \|\kappa\|_\infty, \|\kappa\|_2, \|P\|_\infty$ and $C_2$ depends on $C_\alpha, \alpha, \int \big| \kappa(z) \big| \| z \|^{\alpha } dz$.}

\bigskip

\begin{proof}[Theorem \ref{thm:kde}]
Recall in \ref{K2} $\mathcal{K}:= \{\kappa(\frac{\cdot - x}{h}), x\in\mathcal{X}, h>0\}$ is a uniformly bounded VC-class with dimension $\nu$ and characteristic $\Lambda$. It is simple to see $\|\mathcal{K}\|_\infty$ can be bounded by
$$\|\mathcal{K}\|_\infty \le \|\kappa\|_\infty$$
where $\|\kappa\|_\infty < \infty$ is guaranteed in \ref{K2}.

Then applying Lemma \ref{lm:kde} to $\mathcal{K}$, where $\sigma$ is set to be $\sqrt{\|\kappa\|_2 \|P\|_\infty h^d}$ as provided in Lemma \ref{lm:kde_sq}, gives that there exists a universal constant $C$ not depending on any parameters such that with probability at least $1-\delta$, 
$$
\begin{array}{cl}
     & \displaystyle\sup_{h>0, x\in\mathcal{X}}\big| \hat{P}_h(x) - P_h(x) \big| \\
    \le &  \displaystyle C \bigg( \frac{\nu \|\kappa\|_\infty}{h^d N}\log\big(\frac{2\Lambda\|\kappa\|_\infty}{\|\kappa\|_2 \|P\|_\infty h^d}\big) + \sqrt{\frac{\nu \|\kappa\|_2 \|P\|_\infty}{h^d N}\log\big(\frac{2\Lambda\|\kappa\|_\infty}{\|\kappa\|_2 \|P\|_\infty h^d}\big)} \\
    & \displaystyle + \sqrt{\frac{\|\kappa\|_2 \|P\|_\infty \log(\frac{1}{\delta})}{h^d N }} + \frac{\|\kappa\|_\infty \log (\frac{1}{\delta})}{h^d N}\bigg).
\end{array}
\label{eq:kde_1}
$$

Combining it with Lemma \ref{lm:kde_bias}, we have that with probability at least $1-\delta$,  for any $h > 0$, 
$$\begin{array}{ccl}
    \displaystyle\sup_{x\in\mathcal{X}}\big| \hat{P}_h(x) - P(x) \big| & \le & \displaystyle\sup_{x\in\mathcal{X}}\big| \hat{P}_h(x) - P_h(x) \big| + \displaystyle\sup_{x\in\mathcal{X}}\big| P_h(x) - P(x) \big| \\
     & \le & \displaystyle C_1 \sqrt{\frac{\log(1/h^d) + \log(1/\delta)}{h^d N}} + C_2 h^\alpha,
\end{array} $$
where $C_1$ depends on $\nu, \Lambda, \|\kappa\|_\infty, \|\kappa\|_2, \|P\|_\infty$ and $C_2$ depends on $C_\alpha, \alpha, \int \big| \kappa(z) \big| \| z \|^{\alpha } dz$.
\end{proof}

By setting the bandwidth $h=n^{-\frac{1}{2\alpha+d}}$, the theoretically optimal convergence rate is obtained (See Remark 1). Corollary \ref{crl:dfs-dist} follows from this.

\bigskip

\noindent
\textbf{Corollary \ref{crl:dfs-dist}} \textit{Suppose DFS-Distribution and Deploy are invoked with the bandwidth $h = (n_{\text{dist}})^{-\frac{1}{2\alpha+d}}$. If $\zeta \ge \frac{2\phi}{C_L},$ i.e., $\epsilon \le 250H^2\sqrt{A}C_L\zeta$, 
then with probability at least $1-\delta$,
$$
\|\hat{D}_{\beta} - D_{\beta}\|_\infty \le \frac{\phi}{2C_L}
$$
for all $\beta \in \mathcal{B}$ and 
$$
\|\hat{D}_{R} - D_{R}\|_\infty \le \frac{\phi}{2C_L}
$$
where $C_{\text{dist}}$ is a constant depending on $d, \alpha, \nu, \Lambda, \|\kappa\|_\infty, \|\kappa\|_2, \|P\|_\infty, C_\alpha, \alpha, \int \big| \kappa(z) \big| \| z \|^{\alpha } dz$. Moreover, DFS-Distribution is called at most $HSA$ times, and the total number of episodes collected from simulators is at most $n_{\text{dist}} BHSA$ while the number of episodes collected from the real world is at most $n_{\text{dist}} HSA$.}

\bigskip

\begin{proof}[Corollary \ref{crl:dfs-dist}]
For each $\beta \in \mathcal{B}$ and each $s\in\mathcal{S}$, $D_{\beta,s}$ satisfies the preconditions in Theorem \ref{thm:kde}, so we can apply Theorem \ref{thm:kde} to $D_{\beta, s}, D_{R, s}$. Since $h = (n_{\text{dist}})^{-\frac{1}{2\alpha+d}}$, we have that for a fixed $\beta \in \mathcal{B}$ or the real world, with probability at least $1-\frac{\delta}{(B+1)HSA}$,
$$
\sup_{x\in\mathcal{X}}|\hat{D}_{\beta, s}(x) - D_{\beta, s}(x)| \le  C_{\text{dist}} \cdot (\frac{1}{n_{\text{dist}}})^{\frac{\alpha}{2\alpha+d}}\sqrt{\log(n_{\text{dist}}) + \log\frac{(B+1)HSA}{\delta}} \stackrel{(\dag)}{\le} \frac{\phi}{2C_L}.
$$
Similarly, for the real world, we have that with probability at least $1-\frac{\delta}{(B+1)HSA}$,
$$\sup_{x\in\mathcal{X}}|\hat{D}_{R, s}(x) - D_{R, s}(x)| \le  C_{\text{dist}} \cdot (\frac{1}{n_{\text{dist}}})^{\frac{\alpha}{2\alpha+d}}\sqrt{\log(n_{\text{dist}}) + \log\frac{(B+1)HSA}{\delta}} \stackrel{(\dag)}{\le} \frac{\phi}{2C_L}.$$
$C_{\text{dist}}$ is a constant depending on $d, \alpha, \nu, \Lambda, \|\kappa\|_\infty, \|\kappa\|_2, \|P\|_\infty, C_\alpha, \alpha, \int \big| \kappa(z) \big| \| z \|^{\alpha } dz$. $(\dag)$  is due to (\ref{eq:kernel_phi}).

Suppose the inequalities above hold for all $\beta \in \mathcal{B}$ as well as the real world every time DFS-Distribution is invoked at a particular path (or equivalently a particular state). Because of Assumption \ref{A5} and the precondition that $\zeta \ge \frac{2\phi}{C_L}$, Line 4 in DFS-Distribution can successfully distinguish whether the current path $p$ arrives at a state that is visited before, given that $\epsilon_{\text{dist}} = \frac{\zeta}{2}$. In this case, at most $S$ states are found to be distinct at each layer so that DFS-Distribution is called at most $HSA$ times. Every time DFS-Distribution is called, $B+1$ probability distributions are estimated at Line 3 in DFS-Distribution, so at most $(B+1)HSA$ probability distributions are estimated in total. Obviously, at most $n_{\text{dist}}BHSA$ episodes are collected from simulators and and at most $n_{\text{dist}}HSA$ episodes are collected from the real world. Then taking the union bound of the events that all those estimations are successful, we have that with probability at most $1-\delta$, $\|\hat{D}_{\beta} - D_{\beta}\|_\infty \le \frac{\phi}{2C_L}$ for all $\beta \in \mathcal{B}$ and $\|\hat{D}_{R} - D_{R}\|_\infty \le \frac{\phi}{2C_L}.$
\end{proof}

\newpage

\section{Proof of Theorem \ref{thm:consensus}}
\label{D}

\vskip 0.2in

Recall 
$$\Tilde{V}^f_\beta(p) = \displaystyle \frac{1}{n_{\text{test}}} \sum_{i=1}^{n_{\text{test}}}f(D_{\beta}, x^{(i)}_\beta, \pi^f_{D_{\beta}}(x^{(i)}_\beta)).$$ 
We first prove Lemma \ref{lm:V-hap2V} to bound $|\Tilde{V}^f_\beta(p) - V^f_\beta(p)|, |\hat{V}^f_\beta(p) - V^f_\beta(p)|$, which corresponds to the first conclusion of Theorem \ref{thm:consensus}. The second and third conclusions can be easily proved by using Lemma \ref{lm:V-hap2V}.

\begin{lemma}
Suppose all the preconditions in Theorem \ref{thm:consensus} is satisfied. Then with probability at least $1 - \delta$, 
$$|\Tilde{V}^f_\beta(p) - V^f_\beta(p)| \le \frac{\phi}{2}, \forall f\in\mathcal{F}, \beta \in \mathcal{B},$$
$$|\hat{V}^f_\beta(p) - V^f_\beta(p)| \le \phi, \forall f\in\mathcal{F}, \beta \in \mathcal{B}.$$
\label{lm:V-hap2V}
\end{lemma}

\begin{proof}[Lemma \ref{lm:V-hap2V}]
Note
\begin{equation}
|\hat{V}^f_\beta(p) - V^f_\beta(p)| \le|\hat{V}^f_\beta(p) - \Tilde{V}^f_\beta(p)| + |\Tilde{V}^f_\beta(p) - V^f_\beta(p)|,
\label{eq:consensus.proof0}
\end{equation}
and we are going to bound each term separately. For the first term in (\ref{eq:consensus.proof0}), we have
$$|\hat{V}^f_\beta(p) - \Tilde{V}^f_\beta(p)| \le \displaystyle \frac{1}{n_{\text{test}}} \sum_{i=1}^{n_{\text{test}}} |f (\hat{D}_\beta, x^{(i)}_\beta, \pi^f_{\hat{D}_\beta}(x^{(i)}_\beta)) - f(D_{\beta}, x^{(i)}_\beta, \pi^f_{D_{\beta}}(x^{(i)}_\beta))|.$$
We upper bound $f (\hat{D}_\beta, x^{(i)}_\beta, \pi^f_{\hat{D}_\beta}(x^{(i)}_\beta))$ by
$$\begin{array}{ccl}
    f (\hat{D}_\beta, x^{(i)}_\beta, \pi^f_{\hat{D}_\beta}(x^{(i)}_\beta)) & \stackrel{(\dag)}{\ge} & f (\hat{D}_\beta, x^{(i)}_\beta, \pi^f_{{D}_\beta}(x^{(i)}_\beta)) \\
     & \stackrel{(\ddag)}{\ge} &  f(D_{\beta}, x^{(i)}_\beta, \pi^f_{D_{\beta}}(x^{(i)}_\beta)) - C_L \cdot \|\hat{D}_\beta - D_\beta \|_{\infty},
\end{array}
$$
where $(\dag)$ is due to the definition of $\pi^f$ and $(\ddag)$ is due to Assumption \ref{A4}. 
Similarly we lower bound $f (\hat{D}_\beta, x^{(i)}_\beta, \pi^f_{\hat{D}_\beta}(x^{(i)}_\beta))$ above by
$$\begin{array}{ccl}
    f (\hat{D}_\beta, x^{(i)}_\beta, \pi^f_{\hat{D}_\beta}(x^{(i)}_\beta)) & \le & f(D_{\beta}, x^{(i)}_\beta, \pi^f_{\hat{D}_{\beta}}(x^{(i)}_\beta)) + C_L \cdot \|\hat{D}_\beta - D_\beta \|_{\infty} \\
     & \le &  f(D_{\beta}, x^{(i)}_\beta, \pi^f_{D_{\beta}}(x^{(i)}_\beta)) + C_L \cdot \|\hat{D}_\beta - D_\beta \|_{\infty}.
\end{array}
$$
Therefore, we can bound $|f (\hat{D}_\beta, x^{(i)}_\beta, \pi^f_{\hat{D}_\beta}(x^{(i)}_\beta)) - f(D_{\beta}, x^{(i)}_\beta, \pi^f_{D_{\beta}}(x^{(i)}_\beta))|$ by
\begin{equation}
|f (\hat{D}_\beta, x^{(i)}_\beta, \pi^f_{\hat{D}_\beta}(x^{(i)}_\beta)) - f(D_{\beta}, x^{(i)}_\beta, \pi^f_{D_{\beta}}(x^{(i)}_\beta))| \le C_L \cdot \|\hat{D}_\beta - D_\beta \|_{\infty},
\label{eq:consensus.proof3}
\end{equation}
which implies 
\begin{equation}
|\hat{V}^f_\beta(p) - \Tilde{V}^f_\beta(p)| \le C_L \cdot \|\hat{D}_\beta - D_\beta \|_{\infty}
\label{eq:consensus.proof1}
\end{equation}
for all $f\in\mathcal{F}, \beta \in \mathcal{B}$. Then we bound the second term in (\ref{eq:consensus.proof0}) by the Hoeffding's inequality, that is, with probability at least $1-\delta$, 
\begin{equation}
|\Tilde{V}^f_\beta(p) - V^f_\beta(p)| \le \sqrt{\frac{\log(2FB/\delta)}{2n_{\text{test}}}}, \forall f\in\mathcal{F}, \beta \in \mathcal{B}.    
\label{eq:consensus.proof2}
\end{equation} 
Putting (\ref{eq:consensus.proof0}), (\ref{eq:consensus.proof1}) and (\ref{eq:consensus.proof2}) together, we have 
$$|\hat{V}^f_\beta(p) - V^f_\beta(p)| \le C_L \cdot \|\hat{D}_\beta - D_\beta \|_{\infty} + \sqrt{\frac{\log(2FB/\delta)}{2n_{\text{test}}}}, \forall f\in\mathcal{F}, \beta \in \mathcal{B},$$
$C_L \cdot\|\hat{D}_\beta - D_\beta \|_{\infty}$ can be further bounded by Corollary \ref{crl:dfs-dist} and (\ref{eq:kernel_phi}) and $\sqrt{\frac{\log(2FB/\delta)}{2n_{\text{test}}}}$ can be  bounded by $n_{\text{test}} = \frac{2\log(2FB/\delta)}{\phi^2}$. This gives that with probability at least $1-\delta$,
$$|\Tilde{V}^f_\beta(p) - V^f_\beta(p)| \le \frac{\phi}{2}, \forall f\in\mathcal{F}, \beta \in \mathcal{B},$$
$$|\hat{V}^f_\beta(p) - V^f_\beta(p)| \le \phi, \forall f\in\mathcal{F}, \beta \in \mathcal{B}.$$
\end{proof}

\noindent
\textbf{Theorem \ref{thm:consensus} (Consensus)}
\textit{Suppose the call to DFS-Distribution is successful, that is, (\ref{eq:dfs-dist}) and (\ref{eq:dfs-dist-1}) hold, and Consensus is invoked on path $p$ with $n_{\text{test}} = \frac{2\log(2FB/\delta)}{\phi^2}, \epsilon_{\text{test}} = \tau_1 + 2\phi$, for some $\tau_1 > 0$. Let $\Tilde{V}^f_\beta(p) = \frac{1}{n_{\text{test}}} \displaystyle \sum_{i=1}^{n_{\text{test}}}f(D_{\beta}, x^{(i)}_\beta, \pi^f_{D_{\beta}}(x^{(i)}_\beta))$. Then with probability at least $1 - \delta$, the following statements hold true simultaneously:}

\textit{1. For all $f\in\mathcal{F}, \beta \in \mathcal{B}$, $|\Tilde{V}^f_\beta(p) - V^f_\beta(p)| \le \frac{\phi}{2}, |\hat{V}^f_\beta(p) - V^f_\beta(p)| \le \phi$;}

\textit{2. If $|V^f_\beta(p) - V^g_\beta(p)|\le \tau_1, \forall f, g \in \mathcal{F}, \beta \in \mathcal{B}$, then Consensus returns \textbf{True};}

\textit{3. If Consensus returns \textbf{True}, then $|V^f_\beta(p) - V^g_\beta(p)| \le \epsilon_{\text{test}} + 2\phi, \forall f, g\in \mathcal{F}, \beta \in \mathcal{B}$.}

\bigskip

\begin{proof}[Theorem \ref{thm:consensus}]
According to Lemma \ref{lm:V-hap2V}, with probability at least $1-\delta$,
$$|\Tilde{V}^f_\beta(p) - V^f_\beta(p)| \le \frac{\phi}{2}, \forall f\in\mathcal{F}, \beta \in \mathcal{B},$$
$$|\hat{V}^f_\beta(p) - V^f_\beta(p)| \le \phi, \forall f\in\mathcal{F}, \beta \in \mathcal{B}.$$

If $|V^f_\beta(p) - V^g_\beta(p)|\le \tau_1, \forall f, g \in \mathcal{F}, \beta \in \mathcal{B}$, then in the $1-\delta$ event above, 
$$\begin{array}{ccl}
    |\hat{V}^f_\beta(p) - \hat{V}^g_\beta(p)| & \le & |\hat{V}^f_\beta(p) - V^f_\beta(p)| + |V^f_\beta(p) - V^g_\beta(p)| + |\hat{V}^g_\beta(p) - V^g_\beta(p)| \\
     & \le & \tau_1 + 2\phi.
\end{array}$$
This shows that when $\epsilon_{\text{test}} \ge \tau_1 + 2\phi$, Consensus returns \textbf{True}.

On the other hand, if Consensus returns \textbf{True}, then in the $1-\delta$ event,
$$\begin{array}{ccl}
    |V^f_\beta(p) - V^g_\beta(p)| & \le & |\hat{V}^f_\beta(p) - V^f_\beta(p)| + |\hat{V}^f_\beta(p) - \hat{V}^g_\beta(p)| + |\hat{V}^g_\beta(p) - V^g_\beta(p)| \\
     & \le & \epsilon_{\text{test}} + 2\phi.
\end{array}$$
\end{proof}

\newpage

\section{Proof of Theorem \ref{thm:td-eliminate}}
\label{E}

\vskip 0.2in

Throughout Appendix \ref{E}, we assume that the two preconditions in Theorem \ref{thm:td-eliminate} are met. We prove Theorem \ref{thm:td-eliminate} by showing the following lemmas sequentially.

Recall in Theorem \ref{thm:td-eliminate}, 
$$\hat{V}_\beta^f(p) = \frac{1}{n_{\text{train}}}\displaystyle\sum_{i=1}^{n_{\text{train}}}f(\hat{D}_\beta, x_\beta^{(i)}, \pi_{\hat{D}_\beta}^f(x_\beta^{(i)}))$$
$$\Tilde{V}^f_\beta(p) = \displaystyle \frac{1}{n_{\text{test}}} \sum_{i=1}^{n_{\text{test}}}f(D_{\beta}, x^{(i)}_\beta, \pi^f_{D_{\beta}}(x^{(i)}_\beta)).$$

\begin{lemma}
Suppose all the preconditions in Theorem \ref{thm:td-eliminate} is satisfied. Then
$$\big| f(\hat{D}_\beta, x^{(i)}_\beta, a^{(i)}_\beta) - f(D_\beta, x^{(i)}_\beta, a^{(i)}_\beta) \big| \le \frac{\phi}{2}, \forall f\in\mathcal{F}, \beta \in \mathcal{B},$$
and with probability at least $1 - \frac{\delta}{2}$, 
$$|\Tilde{V}^f_\beta(p) - V^f_\beta(p)| \le \frac{\phi}{2}, \forall f\in\mathcal{F}, \beta\in\mathcal{B},$$
$$|\hat{V}^f_\beta(p) - V^f_\beta(p)| \le \phi, \forall f\in\mathcal{F}, \beta\in\mathcal{B}.$$
\label{lm:0}
\end{lemma}

\begin{proof}[Lemma \ref{lm:0}]
By the same technique to derive (\ref{eq:consensus.proof3}), we have
$$\big| f(\hat{D}_\beta, x^{(i)}_\beta, a^{(i)}_\beta) - f(D_\beta, x^{(i)}_\beta, a^{(i)}_\beta) \big| \le \frac{\phi}{2}, \forall f\in\mathcal{F}, \beta \in \mathcal{B}.$$
Given $n_{\text{train}} = \frac{2\log(4FB/\delta)}{\phi^2}$, by Hoeffding's inequality, with probability at least $1-\frac{\delta}{2}$, for all $f\in\mathcal{F}, \beta\in\mathcal{B}$,
$$|\Tilde{V}^f_\beta(p) - V^f_\beta(p)| \le \sqrt{\frac{1}{2n_{\text{train}}}\log(\frac{4FB}{\delta})} \le \frac{\phi}{2}.$$
Then for all $f\in\mathcal{F}, \beta\in\mathcal{B}$,
$$|\hat{V}^f_\beta(p) - V^f_\beta(p)| \le |\hat{V}^f_\beta(p) - \Tilde{V}^f_\beta(p)| + |\Tilde{V}^f_\beta(p) - V^f_\beta(p)| \le \phi.$$
\end{proof}

Next define the random variable $Y_\beta(f), \hat{Y}_\beta(f)$ for all $f \in \mathcal{F}$ by
$$Y_\beta(f) = (f(D_\beta, x, a)-r-\Tilde{V}_\beta^f(p\circ a))^2 - (f^*(D_\beta, x, a)-r-\Tilde{V}^{f^*}_\beta(p\circ a))^2,$$
where $x\sim D_{\beta, p}$, $a\in \mathcal{A}$ is drawn uniformly and $r\sim R(x, a)$.

\begin{lemma}[\citep{krishnamurthy2016pac}, Lemma 1]
For any $f\in \mathcal{F}, \beta \in \mathcal{B}, a\in \mathcal{A}$ and any $x\in \mathcal{X}$ such that $D_{\beta,p}(x) > 0$, 
$$\mathbb{E}_{r|\beta, x, a}[Y_\beta(f)] = (f_\beta(x,a)-\Tilde{V}_\beta^f(p\circ a) - f^*_\beta(x,a) + V_\beta^{f^*}(p\circ a))^2- (\Tilde{V}_\beta^{f^*}(p\circ a)-V_\beta^{f^*}(p\circ a))^2,$$
$$\mathrm{Var}_{r|\beta, x, a}[Y_\beta(f)] \le 32 \mathbb{E}_{r|\beta, x, a}[Y_\beta(f)]+64\phi^2.$$
\label{lm:1}
\end{lemma}

\begin{lemma}
With probability at least $1 - \frac{\delta}{2}$, $f^*$ is retained by TD-Eliminate and for any surviving $f$,
$$\mathbb{E}_{x,a,r|\beta}[Y_\beta(f)] \le 130\phi^2. $$
\label{lm:2}
\end{lemma}

\begin{proof}[Lemma \ref{lm:2}]
By applying the Bernstein's inequality on $\sum_{i=1}^{n_{\text{train}}}\Big(\mathbb{E}_{x,a,r|\beta}[Y_\beta^{(i)}(f)]-Y_\beta^{(i)}(f)\Big)$, with probability at least $1-\delta$, 
$$\begin{array}{ccl}
    \sum_{i=1}^{n_{\text{train}}}\Big(\mathbb{E}_{x,a,r|\beta}[Y_\beta^{(i)}(f)]-Y_\beta^{(i)}(f)\Big) & \le & \sqrt{2\sum_{i}\mathrm{Var}_{x,a,r|\beta}[Y_\beta^{(i)}(f)]\log(\frac{1}{\delta})} + 6\log(\frac{1}{\delta})  \\
     & \le & \sqrt{64\sum_{i}(\mathbb{E}_{x,a,r|\beta}[Y_\beta^{(i)}(f)] + 2\phi^2)\log(\frac{1}{\delta})} + 6\log(\frac{1}{\delta}).
\end{array}$$
Let $X = \sqrt{\sum_{i}(\mathbb{E}_{x,a,r|\beta}[Y_\beta^{(i)}(f)] + 2\phi^2)}, Z = \displaystyle\sum_iY_\beta^{(i)}(f), C = \sqrt{\log(1/\delta)}$. Then the inequality above is equivalent to 
$$\begin{array}{cl}
     &  X^2-2n_{\text{train}}\phi^2 - Z \le 8XC + 6 C^2 \\
    \implies &   (X-4C)^2 - Z \le 2n_{\text{train}}\phi^2 + 22C^2 \\
    \implies &  -Z \le 2n_{\text{train}}\phi^2 + 22C^2.
\end{array}$$

Recall that in TD-Eliminate, $Risk(f_{\hat{D}_\beta}) = \frac{1}{n_{\text{train}}}\displaystyle \sum_{i = 1}^{n_{\text{train}}}\Big(f(\hat{D}_\beta, x^{(i)}_\beta, a^{(i)}_\beta) - r^{(i)}_\beta - \Hat{V}_\beta^f(p \circ a^{(i)}_\beta)\Big)^2$. Similarly define $Risk(f_{D_\beta}) = \frac{1}{n_{\text{train}}}\displaystyle \sum_{i = 1}^{n_{\text{train}}}\Big(f(D_\beta, x^{(i)}_\beta, a^{(i)}_\beta) - r^{(i)}_\beta - \Tilde{V}_\beta^f(p \circ a^{(i)}_\beta)\Big)^2$. According to the definition of $Z, Y_\beta(f), Risk(f_{D_\beta})$, 
\begin{equation}
Risk(f^*_{D_\beta}) \le Risk(f_{D_\beta}) + 2\phi^2 + \frac{22}{n_{\text{train}}}\log(\frac{1}{\delta}).
\label{eq:td-risk1}
\end{equation}
Next step is to bound $\big|Risk(f_{\hat{D}_\beta}) - Risk(f_{D_\beta})\big|$. For any $f\in\mathcal{F}, \beta\in\mathcal{B}$,
$$\begin{array}{cl}
& \big|\big(f(\hat{D}_\beta, x^{(i)}_\beta, a^{(i)}_\beta) - r^{(i)}_\beta - \Hat{V}_\beta^f(p \circ a^{(i)}_\beta)\big)^2 - \big(f(D_\beta, x^{(i)}_\beta, a^{(i)}_\beta) - r^{(i)}_\beta - \Tilde{V}_\beta^f(p \circ a^{(i)}_\beta)\big)^2\big| \\
=  & \big|\big(f(\hat{D}_\beta, x^{(i)}_\beta, a^{(i)}_\beta) - \Hat{V}_\beta^f(p \circ a^{(i)}_\beta) + f(D_\beta, x^{(i)}_\beta, a^{(i)}_\beta) - \Tilde{V}_\beta^f(p \circ a^{(i)}_\beta) - 2r^{(i)}_\beta\big) \cdot \\
& \big(f(\hat{D}_\beta, x^{(i)}_\beta, a^{(i)}_\beta) - \Hat{V}_\beta^f(p \circ a^{(i)}_\beta) - f(D_\beta, x^{(i)}_\beta, a^{(i)}_\beta) + \Tilde{V}_\beta^f(p \circ a^{(i)}_\beta) \big)\big| \\
\le & 4 \big( \big| f(\hat{D}_\beta, x^{(i)}_\beta, a^{(i)}_\beta) - f(D_\beta, x^{(i)}_\beta, a^{(i)}_\beta) \big| + \big| \Hat{V}_\beta^f(p \circ a^{(i)}_\beta) - \Tilde{V}_\beta^f(p \circ a^{(i)}_\beta) \big| \big) \\
\stackrel{(\dag)}{\le} & 2\phi + 4\big| \Hat{V}_\beta^f(p \circ a^{(i)}_\beta) - \Tilde{V}_\beta^f(p \circ a^{(i)}_\beta) \big| \\ 
\le & 2\phi + 4 \big(\big| \Hat{V}_\beta^f(p \circ a^{(i)}_\beta) - {V}_\beta^f(p \circ a^{(i)}_\beta) \big| + \big| {V}_\beta^f(p \circ a^{(i)}_\beta) - \Tilde{V}_\beta^f(p \circ a^{(i)}_\beta) \big| \big) \\
\stackrel{(\ddag)}{\le} & 8\phi,
\end{array}$$
where $(\dag)$ is due to Lemma \ref{lm:0} and $(\ddag)$ is due to Precondition 1 and 2. Then we can bound $\big|Risk(f_{\hat{D}_\beta}) - Risk(f_{D_\beta})\big|$ by
\begin{equation}
\big|Risk(f_{\hat{D}_\beta}) - Risk(f_{D_\beta})\big| \le 8 \phi.
\label{eq:td-diff-risk}
\end{equation}
(\ref{eq:td-risk1}) and (\ref{eq:td-diff-risk}) together lead to 
$$Risk(f^*_{\hat{D}_\beta}) \le Risk(f_{\hat{D}_\beta}) + 2\phi^2 + 16\phi + \frac{22}{n_{\text{train}}}\log(\frac{1}{\delta}).$$

Taking the union bound over $f\in\mathcal{F}, \beta\in \mathcal{B}$ and replacing $\delta$ with $\frac{\delta}{2FB}$, with probability at least $1-\frac{\delta}{2}$, for all $f\in\mathcal{F}, \beta\in \mathcal{B}$, 
$$Risk(f^*_{\hat{D}_\beta}) \le Risk(f_{\hat{D}_\beta}) + 2\phi^2 + 16\phi + \frac{22}{n_{\text{train}}}\log(\frac{2FB}{\delta}),$$
which implies that for all $\beta\in \mathcal{B}$, 
$$Risk(f^*_{\hat{D}_\beta}) \le \displaystyle\min_{f\in\mathcal{F}} Risk(f_{\hat{D}_\beta}) + 2\phi^2 + 16\phi + \frac{22}{n_{\text{train}}}\log(\frac{2FB}{\delta}).$$

Therefore, with probability at least $1-\frac{\delta}{2}$, $f^*$ is retained by TD-Eliminate.

For any survivor $f$ and any $\beta\in \mathcal{B}$, Line 5 in TD-Eliminate ensures 
$$\begin{array}{ccl}
    Risk(f_{\hat{D}_\beta}) & \le & \displaystyle\min_{f\in\mathcal{F}}  Risk(f_{\hat{D}_\beta}) + 2\phi^2 + 16\phi + \frac{22}{n_{\text{train}}}\log(\frac{2FB}{\delta})  \\
     & \le & \displaystyle Risk(f^*_{\hat{D}_\beta}) + 2\phi^2 + 16\phi + \frac{22}{n_{\text{train}}}\log(\frac{2FB}{\delta}).
\end{array}$$
This together with (\ref{eq:td-diff-risk}) implies
$$ Risk(f_{{D}_\beta}) \le Risk(f^*_{{D}_\beta}) + 2\phi^2 + 32\phi + \frac{22}{n_{\text{train}}}\log(\frac{2FB}{\delta}),$$
which is equivalent to 
$$Z \le 2n_{\text{train}}\phi^2 + 32n_{\text{train}}\phi + 22\log(\frac{2FB}{\delta}).$$
Then 
$$\begin{array}{cl}
     & (X-4C)^2 \le Z + 2n_{\text{train}}\phi^2 + 22 C^2 \le 4n_{\text{train}}\phi^2 + 32 n_{\text{train}}\phi + 44 C^2 \\
    \implies & X^2 \le (\sqrt{4n_{\text{train}}\phi^2 + 32n_{\text{train}}\phi + 44 C^2}+4C)^2 \le 8n_{\text{train}}\phi^2 + 64 n_{\text{train}}\phi + 120\log(\frac{2FB}{\delta}).
\end{array}$$
According to $n_{\text{train}} = \frac{2\log(2FB/\delta)}{\phi^2}$, the definition of $X$ and that $\mathbb{E}_{x,a,r|\beta}[Y_\beta^{(i)}(f)] = \mathbb{E}_{x,a,r|\beta}[Y_\beta(f)]$, we have
$$\mathbb{E}_{x,a,r|\beta}[Y_\beta(f)] \le 6 \phi^2 + 64 \phi + \frac{120}{n_{\text{train}}}\log(\frac{2FB}{\delta}) \le 130 \phi^2. $$
\end{proof}

\begin{lemma}
In the same $1 - \frac{\delta}{2}$ event in Lemma \ref{lm:2}, for any $\beta\in\mathcal{B}$ and $f, g\in \mathcal{F}$ retained by TD-Eliminate, 
$$\big|V_\beta^f(p) - V_\beta^g(p)\big| \le  25\sqrt{A}\phi + \tau_2.$$
\label{lm:4}
\end{lemma}

\begin{proof}[Lemma \ref{lm:4}]
For any $f,g\in\mathcal{F}, \beta\in\mathcal{B}$,
\begin{equation}
    \begin{array}{cl}
    & V_\beta^f(p) - V_\beta^g(p) \\
    = & \mathbb{E}_{x|\beta}[f_{D_\beta}(x, \pi^f_{D_\beta}(x))-g_{D_\beta}(x, \pi^g_{D_\beta}(x))] \\
    \le & \mathbb{E}_{x|\beta}[f_{D_\beta}(x, \pi^f_{D_\beta}(x))-g_{D_\beta}(x, \pi^f_{D_\beta}(x))] \\
    = & \mathbb{E}_{x|\beta}[f_{D_\beta}(x, \pi^f_{D_\beta}(x))-\Tilde{V}_\beta^f(p\circ \pi^f_{D_\beta}(x)) - f^*_{D_\beta}( x,\pi^f_{D_\beta}(x))+V^{f^*}_\beta(p\circ \pi^f_{D_\beta}(x))] \\
     & - \mathbb{E}_{x|\beta}[g_{D_\beta}( x, \pi^f_{D_\beta}(x))-\Tilde{V}^g_\beta(p\circ \pi^f_{D_\beta}(x))-f^*_{D_\beta}(x, \pi^f_{D_\beta}(x))+ V^{f^*}_\beta(p\circ \pi^f_{D_\beta}(x))] \\
     & + \mathbb{E}_{x|\beta}[\Tilde{V}_\beta^f(p\circ \pi^f_{D_\beta}(x)) - \Tilde{V}^g_\beta(p\circ \pi^f_{D_\beta}(x)) ].
\end{array}
\label{eq:lm24-1}
\end{equation}
We bound $|V_\beta^f(p) - V_\beta^g(p)|$ by bounding the three expectation terms separately. First, Lemma \ref{lm:1} implies that
$$\begin{array}{cl}
     & \mathbb{E}_{a,r|x,\beta}[Y_\beta(f)]+\mathbb{E}_{a|x,\beta}[\big(\Tilde{V}^{f^*}_\beta(p\circ a)-V_\beta^{f^*}(p\circ a)\big)^2] \\
     = & \mathbb{E}_{a|x,\beta}[\big(f_{D_\beta}( x, a)-\Tilde{V}_\beta^f(p\circ a) - f^*_{D_\beta}( x,a)+V^{f^*}_\beta(p\circ a)\big)^2] \\
     \ge & \frac{1}{A}\cdot\big(f_{D_\beta}( x, \pi^f_{D_\beta}(x))-\Tilde{V}_\beta^f(p\circ \pi^f_{D_\beta}(x)) - f^*_{D_\beta}( x,\pi^f_{D_\beta}(x))+V^{f^*}_\beta(p\circ \pi^f_{D_\beta}(x))\big)^2.
\end{array}$$
Hence we can bound the first expectation term in (\ref{eq:lm24-1}) by
\begin{equation}
\begin{array}{cl}
     & \mathbb{E}_{x|\beta}[f_{D_\beta}(x, \pi^f_{D_\beta}(x))-\Tilde{V}_\beta^f(p\circ \pi^f_{D_\beta}(x)) - f^*_{D_\beta}(x,\pi^f_{D_\beta}(x))+V^{f^*}_\beta(p\circ \pi^f_{D_\beta}(x))] \\
    \le & \sqrt{ \mathbb{E}_{x|\beta}[\big( f_{D_\beta}(x, \pi^f_{D_\beta}(x))-\Tilde{V}_\beta^f(p\circ \pi^f_{D_\beta}(x)) - f^*_{D_\beta}(x,\pi^f_{D_\beta}(x))+V^{f^*}_\beta(p\circ \pi^f_{D_\beta}(x))\big)^2]} \\
    \le & \sqrt { A \cdot \mathbb{E}_{x|\beta}\big[\mathbb{E}_{a,r|x,\beta}[Y_\beta(f)]+\mathbb{E}_{a|x,\beta}[\big(\Tilde{V}^{f^*}_\beta(p\circ a)-V_\beta^{f^*}(p\circ a)\big)^2]\big] } \\
    \stackrel{(\dag)}{\le} &  \sqrt { A \cdot \big(\mathbb{E}_{x,a,r|\beta}[Y_\beta(f)]+\phi^2 \big) } \\
    \stackrel{(\ddag)}{\le} & \sqrt {A \cdot 130\phi^2}.
\end{array}
\label{eq:lm24-2}
\end{equation}
$(\dag)$ is due to Precondition 1 and $(\ddag)$ holds true in the $1-\frac{\delta}{2}$ event in Lemma \ref{lm:2}. Similarly, the second expectation term in (\ref{eq:lm24-1}) can be bounded in the same $1-\frac{\delta}{2}$ event by
\begin{equation}
\begin{array}{cl}
     & \mathbb{E}_{x|\beta}[g_{D_\beta}(x, \pi^f_{D_\beta}(x))-\Tilde{V}^g_\beta(p\circ \pi^f_{D_\beta}(x))-f^*_{D_\beta}(x, \pi^f_{D_\beta}(x))+ V^{f^*}_\beta(p\circ \pi^f_{D_\beta}(x))] \\
    \le & \sqrt {A \cdot 130\phi^2}.
\end{array}
\label{eq:lm24-3}
\end{equation}
Lastly, we bound the third expectation term by
\begin{equation}
\begin{array}{cl}
     & \mathbb{E}_{x|\beta}[\Tilde{V}_\beta^f(p\circ \pi^f_{D_\beta}(x)) - \Tilde{V}^g_\beta(p\circ \pi^f_{D_\beta}(x)) ] \\
    \le & \mathbb{E}_{x|\beta}[\big|\Tilde{V}_\beta^f(p\circ \pi^f_{D_\beta}(x)) - {V}^f_\beta(p\circ \pi^f_{D_\beta}(x))\big|] + \mathbb{E}_{x|\beta}[\big|{V}_\beta^f(p\circ \pi^f_{D_\beta}(x)) - {V}^g_\beta(p\circ \pi^f_{D_\beta}(x))\big|] \\
    & + \mathbb{E}_{x|\beta}[\big|{V}_\beta^g(p\circ \pi^f_{D_\beta}(x)) - \Tilde{V}^g_\beta(p\circ \pi^f_{D_\beta}(x))\big|] \\
    \stackrel{(\perp)}{\le} & \phi + \tau_2.
\end{array}
\label{eq:lm24-4}
\end{equation}
$(\perp)$ holds true due to Precondition 1. (\ref{eq:lm24-1}), (\ref{eq:lm24-2}), (\ref{eq:lm24-3}) and (\ref{eq:lm24-4}) together give that in the $1-\frac{\delta}{2}$ event in Lemma \ref{lm:2}, for any $f,g\in\mathcal{F}, \beta\in\mathcal{B}$,
$$\big|V_\beta^f(p) - V_\beta^g(p)\big| \le 2 \sqrt{130 A\phi^2} + \phi + \tau_2 \le 25\sqrt{A}\phi + \tau_2.$$
\end{proof}

\begin{lemma}
In the same $1 - \frac{\delta}{2}$ event in Lemma \ref{lm:2}, for any $\beta\in\mathcal{B}$ and $f\in \mathcal{F}$ retained by TD-Eliminate, 
$$V^{f^*}_\beta(p) - V^{f^*}_\beta(p, \pi^f_{\hat{D}_\beta}) \le 25\sqrt{A}\phi + 2\tau_2.$$
\label{lm:5}
\end{lemma}

\begin{proof}[Lemma \ref{lm:5}]
$$\begin{array}{cl}
& V^{f^*}_\beta(p) - V^{f^*}_\beta(p, \pi^f_{\hat{D}_\beta}) \\
= & \mathbb{E}_{x\sim D_{\beta, p}}[f^*_{D_\beta}( x, \pi^{f^*}_{D_\beta}(x)) - f^*_{D_\beta}( x, \pi^f_{\hat{D}_\beta}(x))] \\
\le & \mathbb{E}_{x\sim D_{\beta, p}}[f^*_{D_\beta}( x, \pi^{f^*}_{D_\beta}(x)) - f_{\hat{D}_\beta}( x, \pi^{f^*}_{D_\beta}(x)) + f_{\hat{D}_\beta}( x, \pi^{f}_{\hat{D}_\beta}(x)) - f^*_{D_\beta}(x, \pi^f_{\hat{D}_\beta}(x))].
\end{array}$$
Next define
$$W_{\beta, x, a} = f_{\hat{D}_\beta}( x, a) - \Tilde{V}^f_{\beta}(p\circ a) - f^*_{D_\beta}( x, a) + V^{f^*}_{\beta}(p\circ a).$$
Then we rewrite
\begin{equation}
\begin{array}{cl}
     & V^{f^*}_\beta(p) - V^{f^*}_\beta(p, \pi^f_{\hat{D}_\beta}) \\
    \le & \mathbb{E}_{x\sim D_{\beta, p}}[W_{\beta, x, \pi^f_{\hat{D}_\beta}(x)} - W_{\beta, x, \pi^{f^*}_{D_\beta}(x)} + \Tilde{V}^f_{\beta}(p\circ \pi^f_{\hat{D}_\beta}(x)) - V^{f^*}_{\beta}(p\circ \pi^f_{\hat{D}_\beta}(x)) \\
    & - \Tilde{V}^f_{\beta}(p\circ \pi^{f^*}_{D_\beta}(x)) + V^{f^*}_{\beta}(p\circ \pi^{f^*}_{D_\beta}(x)) ].
\end{array}
\label{eq:lm25-1}
\end{equation}
By the same technique to derive (\ref{eq:lm24-2}) and (\ref{eq:lm24-3}), $\mathbb{E}_{x\sim D_{\beta, p}}[W_{\beta, x, a}]$ can be bounded by
\begin{equation}
\mathbb{E}_{x\sim D_{\beta, p}}[W_{\beta, x, a}] \le \sqrt{130A\phi^2}.
\label{eq:lm25-2}
\end{equation}
Then according to Precondition 1 and 3,
\begin{equation}
\begin{array}{cl}
& \Tilde{V}^f_{\beta}(p\circ \pi^f_{\hat{D}_\beta}(x)) - V^{f^*}_{\beta}(p\circ \pi^f_{\hat{D}_\beta}(x)) - \Tilde{V}^f_{\beta}(p\circ \pi^{f^*}_{D_\beta}(x)) + V^{f^*}_{\beta}(p\circ \pi^{f^*}_{D_\beta}(x)) \\
\le & \big| {V}^f_{\beta}(p\circ \pi^f_{\hat{D}_\beta}(x)) - V^{f^*}_{\beta}(p\circ \pi^f_{\hat{D}_\beta}(x)) - {V}^f_{\beta}(p\circ \pi^{f^*}_{D_\beta}(x)) + V^{f^*}_{\beta}(p\circ \pi^{f^*}_{D_\beta}(x)) \big| + \phi \\
\le & 2\tau_2 + \phi.
\end{array}
\label{eq:lm25-3}
\end{equation}
(\ref{eq:lm25-1}), (\ref{eq:lm25-2}) and (\ref{eq:lm25-3}) together give 
$$V^{f^*}_\beta(p) - V^{f^*}_\beta(p, \pi^f_{\hat{D}_\beta}) \le 2 \sqrt{130A\phi^2} + 2\tau_2 + \phi \le 25\sqrt{A}\phi + 2\tau_2.$$
\end{proof}

\noindent
\textbf{Theorem \ref{thm:td-eliminate}}
\textit{Suppose the call to DFS-Distribution is successful, that is, (\ref{eq:dfs-dist}) and (\ref{eq:dfs-dist-1}) hold, and TD-Eliminate is invoked at path $p$ with $\mathcal{F}, \phi, \delta$ and $n_{\text{train}} = \frac{2\log(4FB/\delta)}{\phi^2}$, and the following conditions are satisfied:}

\textit{(Precondition 1): $|\Tilde{V}^f_\beta(p\circ a) - V^f_\beta(p\circ a)| \le \frac{\phi}{2}, \forall f \in \mathcal{F}, a\in \mathcal{A}, \beta\in \mathcal{B};$}

\textit{(Precondition 2): $|\hat{V}^f_\beta(p\circ a) - V^f_\beta(p\circ a)| \le \phi, \forall f \in \mathcal{F}, a\in \mathcal{A}, \beta\in \mathcal{B};$}

\textit{(Precondition 3): $|V^f_\beta(p\circ a) - V^g_\beta(p\circ a)| \le \tau_2, \forall f,g \in \mathcal{F}, a\in \mathcal{A}, \beta\in \mathcal{B}.$}

\noindent
\textit{Let $\hat{V}_\beta^f(p) = \frac{1}{n_{\text{train}}}\displaystyle\sum_{i=1}^{n_{\text{train}}}f(\hat{D}_\beta, x_\beta^{(i)}, \pi_{\hat{D}_\beta}^f(x_\beta^{(i)}))$ and $\Tilde{V}_\beta^f(p) = \frac{1}{n_{\text{train}}}\displaystyle\sum_{i=1}^{n_{\text{train}}}f({D}_\beta, x_\beta^{(i)}, \pi_{{D}_\beta}^f(x_\beta^{(i)}))$. Then with probability at least $1-\delta$, the following hold simultaneously:}

\textit{1. $f^*$ is retained;}

\textit{2. For all $f\in\mathcal{F}$ retained and for all $\beta\in\mathcal{B}$, $|\Tilde{V}^f_\beta(p) - V^f_\beta(p)| \le \frac{\phi}{2}, |\hat{V}^f_\beta(p) - V^f_\beta(p)| \le \phi$;}

\textit{3. For all $f,g\in\mathcal{F}$ retained and for all $\beta\in\mathcal{B}$, $\big|V_\beta^f(p) - V_\beta^g(p)\big|  \le   25\sqrt{A}\phi + \tau_2$;}

\textit{4. For all $f\in\mathcal{F}$ retained and for all $\beta\in\mathcal{B}$, $V^{f^*}_\beta(p) - V^{f^*}_\beta(p, \pi^f_{\hat{D}_{\beta}}) \le 25\sqrt{A}\phi + 2\tau_2$.}

\bigskip

\begin{proof}[Theorem \ref{thm:td-eliminate}] 
By taking the union bound of lemmas, we complete the proof.
\end{proof}

\newpage

\section{Proof of Theorem \ref{thm:dfs-learn}}
\label{F}

\vskip 0.2in

\noindent
\textbf{Theorem \ref{thm:dfs-learn}}
\textit{Suppose the call to DFS-Distribution is successful, that is, (\ref{eq:dfs-dist}) and (\ref{eq:dfs-dist-1}) hold, and DFS-Learn is invoked at path $p$ with $\mathcal{F},\delta,\phi$. With probability at least $1-\delta$, for any $h=1,2,...,H$ and any $s_h\in\mathcal{S}_h$ such that TD-Eliminate is called, the conclusions of Theorem \ref{thm:td-eliminate} hold with $\tau_2 = (H-h)(25\sqrt{A}\phi)$. Moreover, the number of episodes executed on similators by DFS-Learn is at most
$$\mathcal{O}(\frac{HSAB}{\phi^2}\log(\frac{HSAFB}{\delta})).$$}

\begin{proof}[Theorem \ref{thm:dfs-learn}] 
In order to show that the conclusions of Theorem \ref{thm:td-eliminate} hold with $\tau_2 = (H-h)(25\sqrt{A}\phi)$, it suffices to show that if all calls to Consensus and TD-Eliminate are successful (i.e., Theorem \ref{thm:consensus} and Theorem \ref{thm:td-eliminate} do not fall into the $\delta$ failure cases), then the preconditions of Theorem \ref{thm:td-eliminate} are satisfied. By taking the union bound of the events of successful calls to Consensus and TD-Eliminate, we can compute the sample complexity.

We use an induction over $h$ (from $H$ to $1$) to show the satisfaction of the preconditions.

\noindent
\textbf{Inductive Claim}: The preconditions of Theorem \ref{thm:td-eliminate} are satisfied with $\tau_2 = (H-h)(25\sqrt{A}\phi)$.

\noindent
\textbf{Inductive Base}: When $h = H$, there is only one step left so that for any $a\in \mathcal{A}, f\in\mathcal{F}, \beta\in\mathcal{B}$ and any $p$ such that $p$ arrives at some $s_H\in\mathcal{S}_H$, we have $\hat{V}^f_\beta(p\circ a) = V^f_\beta(p\circ a) = 0$. Hence the claim holds when $h = H$.

\noindent
\textbf{Inductive Hypothesis}: The claim holds for $h+1$.

\noindent
\textbf{Inductive Step}: For any $a\in \mathcal{A}, f\in\mathcal{F}, \beta\in\mathcal{B}$ and any $p$ such that $p$ arrives at some $s_h\in\mathcal{S}_h$, we know that $p\circ a$ arrives at some $s_{h+1}\in\mathcal{S}_{h+1}$. We know that for each $p\circ a$, there are two cases: either Consensus returns \textit{True} for $p\circ a$ or TD-Eliminate is invoked on $p\circ a$. If TD-Eliminate is invoked on $p\circ a$, due to the inductive hypothesis, we know that the preconditions of Theorem \ref{thm:td-eliminate} are satisfied for $p\circ a$ with $\tau_2 = (H-h-1)(25\sqrt{A}\phi)$. Then, according to Theorem \ref{thm:td-eliminate}, we have $|\Tilde{V}^f_\beta(p\circ a) - V^f_\beta(p\circ a)| \le \frac{\phi}{2}$, $|\hat{V}^f_\beta(p\circ a) - V^f_\beta(p\circ a)| \le \phi$ and $|V^f_\beta(p\circ a) - V^g_\beta(p\circ a)| \le 25 \sqrt{ A }\phi + \tau_2 \le (H-h)( 25 \sqrt{A}\phi)$. If Consensus returns \textit{True} for $p\circ a$, given that we set $\tau_1 = 21 \sqrt{ A }\phi + \tau_2$ implicitly in DFS-Learn. then $|\Tilde{V}^f_\beta(p\circ a) - V^f_\beta(p\circ a)| \le \frac{\phi}{2}$, $|\hat{V}^f_\beta(p\circ a) - V^f_\beta(p\circ a)| \le \phi$ and $|V^f_\beta(p\circ a) - V^g_\beta(p\circ a)| \le \epsilon_{\text{test}} + 2\phi$ where $\epsilon_{\text{test}} = \tau_1 + 2\phi$.  Then $|V^f_\beta(p\circ a) - V^g_\beta(p\circ a)| \le \tau_1 + 4\phi \le (H-h)(25\sqrt{A}\phi)$. Combining these two cases, the preconditions of Theorem \ref{thm:td-eliminate} are satisfied for $p$ with $\tau_2 = (H-h)(25\sqrt{A}\phi)$.

The next part is to determine the sample complexity. Since $\tau_1 = 21 \sqrt{ A }\phi + \tau_2$, Consensus returns \textit{True} at any state at which TD-Eliminate is already called. Therefore, for each $h$, TD-Eliminate is invoked for at most $S$ times and consequently Consensus is invoked for at most $SA$ times, which implies that TD-Eliminate and Consensus is invoked for at most $HS$ and $HSA$ times, respectively, in total. Then, we take the union bound over the calls to TD-Eliminate and Consensus by replacing the $\delta$ in Theorem \ref{thm:consensus} with $\frac{\delta}{2HSA}$ and the $\delta$ in Theorem \ref{thm:td-eliminate} with $\frac{\delta}{2HS}$. Hence the number of episodes is at most
$$H(SB\cdot n_{\text{train}} + SAB\cdot n_{\text{test}}) \le \mathcal{O}(\frac{HSAB}{\phi^2}\log(\frac{HSAFB}{\delta})). $$
\end{proof}

\section{Proof of Theorem \ref{thm:boundK} \& Corollary \ref{crl:1}}
\label{G}

\vskip 0.2in

\noindent
\textbf{Theorem \ref{thm:boundK}}
\textit{Suppose the call to DFS-Distribution is successful, that is, (\ref{eq:dfs-dist}) and (\ref{eq:dfs-dist-1}) hold, and $B = \frac{\log(4F/\delta')}{2\phi^2}$. Then with probability at least $1-\delta'$, for all $f \in \mathcal{F}$,
$$|\Bar{V}(\pi^f_{\hat{D}}) - V(\pi^f_{\hat{D}})| \le \phi,$$
$$|\Bar{V}^f(\emptyset) - V^f(\emptyset)| \le \phi.$$}

\begin{proof}[Theorem \ref{thm:boundK}]
This can be easily proved by applying Hoeffding's inequality and taking the union bound.
\end{proof}

\noindent
\textbf{Corollary \ref{crl:1}} 
\textit{Suppose the call to DFS-Distribution is successful, that is, (\ref{eq:dfs-dist}) and (\ref{eq:dfs-dist-1}) hold. Then with probability at least $1-\delta-\delta'$, $\hat{V}^*$ in Line 6 in Sim2Real satisfies
$$|\hat{V}^* - V^*| \le 33H\sqrt{A}\phi.$$}

\begin{proof}[Corollary \ref{crl:1}]
Let $f,g \in\mathcal{F}$ be arbitrary survivors of DFS-Learn$(\emptyset, \mathcal{B}, \mathcal{F}, \hat{D}, \phi, \delta/3)$ in Line 5 in Sim2Real. Combining Theorem \ref{thm:dfs-learn} and Theorem \ref{thm:boundK}, with probability at least $1-\delta -\delta'$,
\begin{equation}
\begin{array}{ccl}
    |\hat{V}^f(\emptyset) - V^f(\emptyset)| & \le & |\hat{V}^f(\emptyset) - \Bar{V}^f(\emptyset)| + |\Bar{V}^f(\emptyset) - V^f(\emptyset)| \\
     & \le & \displaystyle\frac{1}{B}\sum_{\beta\in\mathcal{B}}|\hat{V}_\beta^f(\emptyset) - V_\beta^f(\emptyset)| + |\Bar{V}^f(\emptyset) - V^f(\emptyset)| \\
     & \le & 2\phi.
\end{array}
\label{eq:crl9.1}
\end{equation}
On the other hand, in the $1-\delta -\delta'$ event above, by Theorem \ref{thm:dfs-learn},
$$\begin{array}{ccl} 
|\hat{V}^f_\beta(\emptyset) - \hat{V}^g_\beta(\emptyset)| & \le & |\hat{V}^f_\beta(\emptyset) - V^f_\beta(\emptyset)| + |V^f_\beta(\emptyset) - V^g_\beta(\emptyset)| + |V^g_\beta(\emptyset) - \hat{V}^g_\beta(\emptyset)| \\ & \le & 25H\sqrt{A}\phi + 2\phi,
\end{array}$$ 
and consequently,
\begin{equation}
\begin{array}{ccl} 
|\hat{V}^f(\emptyset) - \hat{V}^g(\emptyset)| & = & \Big|\displaystyle\frac{1}{B}\sum_{\beta\in\mathcal{B}}\hat{V}^f_\beta(\emptyset) - \frac{1}{B}\sum_{\beta\in\mathcal{B}}\hat{V}^g_\beta(\emptyset)\Big| \\ 
& \le & \displaystyle \frac{1}{B}\sum_{\beta\in \mathcal{B}}|\hat{V}^f_\beta(\emptyset) - \hat{V}^g_\beta(\emptyset)| \\
& \le & 25H\sqrt{A}\phi + 2\phi. 
\end{array}
\label{eq:crl9.2}
\end{equation}
(\ref{eq:crl9.1}) and (\ref{eq:crl9.2}) imply 
$$\begin{array}{ccl} 
|V^f(\emptyset)-V^g(\emptyset)| & \le & |V^f(\emptyset)-\hat{V}^f(\emptyset)| + |\hat{V}^f(\emptyset)-\hat{V}^g(\emptyset)| + |\hat{V}^g(\emptyset)-V^g(\emptyset)| \\ 
& \le & 25H\sqrt{A}\phi + 6\phi.
\end{array}$$

Since $\hat{V}^* = \hat{V}^f(\emptyset)$ for some surviving $f \in\mathcal{F}$ and $f^*$ is retained in the $1-\delta-\delta'$ event above,\
$$\begin{array}{ccl}
     |\hat{V}^* - V^*| & \le & |\hat{V}^f(\emptyset) - V^f(\emptyset)| + |V^f(\emptyset) - V^{f^*}(\emptyset)|  \\
     & \le & 25H\sqrt{A}\phi + 8\phi \\
     & \le & 33H\sqrt{A}\phi.
\end{array}$$
\end{proof}

\newpage

\section{Proof of Theorem \ref{thm:Learn-on-Simulators}}
\label{H}

\vskip 0.2in

Define the set of states that are learned via DFS-Learn as
$$\displaystyle L = \{s\in\mathcal{S}: \max_{f\in\mathcal{F}, \beta\in\mathcal{B}} \Big( V^*_\beta(s) - V^{f^*}_\beta(s, \pi^f_{\hat{D}_\beta}) \Big) \le 25\sqrt{A}\phi + 50(H-h)\sqrt{A}\phi\}.$$
Let $\Bar{L}$ be the complement of $L$.
For any $s \in L$, we use $\mathbb{P}(s, \pi^f_{\hat{D}_\beta} \to \Bar{L})$ to denote the probability of visiting some state in $\Bar{L}$ from $s$ with policy $\pi^f_{\hat{D}_\beta}$.

In this section, we first assume all calls to DFS-Learn are successful (i.e., the conclusions in Theorem \ref{thm:dfs-learn}, Theorem \ref{thm:boundK} and Corollary \ref{crl:1} hold), and we compute the corresponding probability in the end.

\begin{lemma} 
Suppose all calls to DFS-Learn are successful. Then for any surviving $f \in \mathcal{F}$, 
$$V^* - V(\pi^f_{\hat{D}}) \le 77 H^2\sqrt{A}\phi + \frac{1}{B}\displaystyle\sum_{\beta\in\mathcal{B}}\mathbb{P}(s_1, \pi^f_{\hat{D}_\beta} \to \Bar{L}).$$
\label{lm:8}
\end{lemma}

\begin{proof}[Lemma \ref{lm:8}]
First, we prove the following claim by induction.

\noindent
\textbf{Inductive Claim}: For all $\beta\in\mathcal{B}, h=1,...,H$ and all $s \in \mathcal{S}_h \cap L$, 
\begin{equation}
V^*_\beta(s) - V_\beta(s, \pi^f_{\hat{D}_\beta}) \le 75 (H-h+1)^2\sqrt{A}\phi + \mathbb{P}(s, \pi^f_{\hat{D}_\beta} \to \Bar{L}).
\label{eq:thm10.0}
\end{equation}

\noindent
\textbf{Inductive Base}: When $h=H+1$ (this means the agent has finished all $H$ actions), there is zero future reward so that the claim holds.

\noindent
\textbf{Inductive Hypothesis}: The claim holds for $h+1$.

\noindent
\textbf{Inductive Step}: For any $s \in \mathcal{S}_h\cap L$,
\begin{equation}
V^*_\beta(s) - V_\beta(s, \pi^f_{\hat{D}_\beta}) = V^*_\beta(s) - V^{f^*}_\beta(s, \pi^f_{\hat{D}_\beta}) + V^{f^*}_\beta(s, \pi^f_{\hat{D}_\beta}) - V_\beta(s, \pi^f_{\hat{D}_\beta}).
\label{eq:thm10.1}
\end{equation}
According to Theorem \ref{thm:dfs-learn},
\begin{equation}
V^*_\beta(s) - V^{f^*}_\beta(s, \pi^f_{\hat{D}_\beta}) \le 50(H-h)\sqrt{A}\phi + 25\sqrt{A}\phi \le 75(H-h)\sqrt{A}\phi.
\label{eq:thm10.2}
\end{equation}
We bound $V^{f^*}_\beta(s, \pi^f_{\hat{D}_\beta}) - V_\beta(s, \pi^f_{\hat{D}_\beta})$ by
\begin{equation}
\begin{array}{cl}
& V^{f^*}_\beta(s, \pi^f_{\hat{D}_\beta}) - V_\beta(s, \pi^f_{\hat{D}_\beta}) \\
= & \mathbb{E}_{x\sim D_{\beta, s}}\big[ V^*_\beta(s\circ \pi^f_{\hat{D}_\beta}(x)) - V_\beta(s\circ \pi^f_{\hat{D}_\beta}(x), \pi^f_{\hat{D}_\beta}) \big] \\
\le & [ V^*_\beta(s\circ \pi^f_{\hat{D}_\beta}(x)) - V_\beta(s\circ \pi^f_{\hat{D}_\beta}(x), \pi^f_{\hat{D}_\beta}) \big] \mathbb{P}(s\circ \pi^f_{\hat{D}_\beta}(x) \in L) + \mathbb{P}(s\circ \pi^f_{\hat{D}_\beta}(x) \notin L) \\
\stackrel{(\dag)}{\le} & \big[ 75 (H-h)^2\sqrt{A}\phi + \mathbb{P}(s\circ \pi^f_{\hat{D}_\beta}(x), \pi^f_{\hat{D}_\beta} \to \Bar{L}) \big]\mathbb{P}(s\circ \pi^f_{\hat{D}_\beta}(x) \in L) + \mathbb{P}(s\circ \pi^f_{\hat{D}_\beta}(x) \notin L) \\
\stackrel{(\ddag)}{\le} & 75 (H-h)^2\sqrt{A}\phi + \mathbb{P}(s, \pi^f_{\hat{D}_\beta} \to \Bar{L}).
\end{array}
\label{eq:thm10.3}
\end{equation}
$(\dag)$ is based on the inductive hypothesis and $(\ddag)$ is due to the fact that 
$$\mathbb{P}(s\circ \pi^f_{\hat{D}_\beta}(x), \pi^f_{\hat{D}_\beta} \to \Bar{L})\mathbb{P}(s\circ \pi^f_{\hat{D}_\beta}(x) \in L) + \mathbb{P}(s\circ \pi^f_{\hat{D}_\beta}(x) \notin L) = \mathbb{P}(s, \pi^f_{\hat{D}_\beta} \to \Bar{L}).$$
Therefore, combining (\ref{eq:thm10.1}), (\ref{eq:thm10.2}) and (\ref{eq:thm10.3}), we have
$$\begin{array}{ccl}
    V^*_\beta(s) - V_\beta(s, \pi^f_{\hat{D}_\beta}) & \le & 75(H-h)\sqrt{A}\phi + 75 (H-h)^2\sqrt{A}\phi + \mathbb{P}(s, \pi^f_{\hat{D}_\beta} \to \Bar{L}) \\
     & \le & 75 (H-h+1)^2\sqrt{A}\phi + \mathbb{P}(s, \pi^f_{\hat{D}_\beta} \to \Bar{L})
\end{array}$$
Hence, we proved the claim.

\bigskip
(\ref{eq:thm10.0}) in the claim directly implies 
$$V^*_\beta - V_\beta(\pi^f_{\hat{D}_\beta}) \le 75 H^2\sqrt{A}\phi + \mathbb{P}(s_1, \pi^f_{\hat{D}_\beta} \to \Bar{L}).$$
Recall in Theorem \ref{thm:boundK}, for all surviving $f \in\mathcal{F}$,
$$|\Bar{V}(\pi^f_{\hat{D}}) - V(\pi^f_{\hat{D}})| \le \phi,$$
$$|\Bar{V}^f(\emptyset) - V^f(\emptyset)| \le \phi,$$
where $\Bar{V}^f(\emptyset):= \frac{1}{B}\sum_{\beta\in\mathcal{B}}V_\beta^f(\emptyset)$ and $\Bar{V}(\pi^f_{\hat{D}}) = \frac{1}{B}\sum_{\beta\in\mathcal{B}}V_\beta(\pi^f_{\hat{D}_\beta})$. Since $f^*$ is retained in $\mathcal{F}$,
$$|\Bar{V}^* - V^*| \le \phi.$$
Then 
$$\begin{array}{ccl}
V^* - V(\pi^f_{\hat{D}}) & \le & |V^* - \Bar{V}^*| + |\Bar{V}^* - \Bar{V}(\pi^f_{\hat{D}})| + |\Bar{V}(\pi^f_{\hat{D}}) - V(\pi^f_{\hat{D}})|  \\
& \le & 2\phi + \frac{1}{B}\displaystyle\sum_{\beta\in\mathcal{B}}|V^*_\beta - V_\beta(\pi^f_{\hat{D}_\beta})| \\
& \le &  2\phi + 75 H^2\sqrt{A}\phi + \frac{1}{B}\displaystyle\sum_{\beta\in\mathcal{B}}\mathbb{P}(s_1, \pi^f_{\hat{D}_\beta} \to \Bar{L}) \\
& \le & 77 H^2\sqrt{A}\phi + \frac{1}{B}\displaystyle\sum_{\beta\in\mathcal{B}}\mathbb{P}(s_1, \pi^f_{\hat{D}_\beta} \to \Bar{L}).
\end{array}$$
\end{proof}

\begin{lemma} 
Suppose all calls to DFS-Learn are successful and $\phi = \frac{\epsilon}{500H^2\sqrt{A}}, n_1 = \frac{32\log(2B/\delta'')}{\epsilon^2}$. If $f\in\mathcal{F}$ is selected and $\pi^f$ does not satisfy $|\hat{V}^* - \hat{V}(\pi^f_{\hat{D}})| \le \epsilon_{\text{demand}}$, then with probability at least $1 - \delta'' - \exp (-\epsilon n_2 B / 8)$, at least one of the $n_2 B$ trajectories visit some state in $\Bar{L}$.
\label{lm:9}
\end{lemma}

\begin{proof}[Lemma \ref{lm:9}]
\begin{equation}
|\hat{V}^* - \hat{V}(\pi^f_{\hat{D}})| \le |\hat{V}^* - {V}^*| + |{V}^* - {V}(\pi^f_{\hat{D}})| + |{V}(\pi^f_{\hat{D}}) - \hat{V}(\pi^f_{\hat{D}})|.
\label{eq:thm10.4}
\end{equation}
Then we bound each term separately. For the first term, according to Corollary \ref{crl:1},
\begin{equation}
|\hat{V}^* - {V}^*| \le 33H\sqrt{A}\phi.
\label{eq:thm10.5}
\end{equation}
For the second term, by Lemma \ref{lm:8},
\begin{equation}
|{V}^* - {V}(\pi^f_{\hat{D}})| \le 77 H^2\sqrt{A}\phi + \frac{1}{B}\displaystyle\sum_{\beta\in\mathcal{B}}\mathbb{P}(s_1, \pi^f_{\hat{D}_\beta} \to \Bar{L}).
\label{eq:thm10.6}
\end{equation}
For the last term, due to Theorem \ref{thm:boundK},
\begin{equation}
|{V}(\pi^f_{\hat{D}}) - \hat{V}(\pi^f_{\hat{D}})|  \le  |{V}(\pi^f_{\hat{D}}) - \Bar{V}(\pi^f_{\hat{D}})| + |\Bar{V}(\pi^f_{\hat{D}}) - \hat{V}(\pi^f_{\hat{D}})| \le  \phi + |\Bar{V}(\pi^f_{\hat{D}}) - \hat{V}(\pi^f_{\hat{D}})|
\label{eq:thm10.7}
\end{equation}
Since, by Hoeffding's inequality, with probability at least $1-\delta''$, for all $\theta\in\mathcal{B}$,
$$\Big|\frac{1}{n_1}\sum_{j=1}^{n_1}v^{(j)}_\beta - V_\beta(\pi^f_{\hat{D}_\beta})\Big| \le \sqrt{\frac{1}{2n_1}\log(\frac{2K}{\delta''}}),$$
we have
\begin{equation}
\begin{array}{ccl}
     |\Bar{V}(\pi^f_{\hat{D}}) - \hat{V}(\pi^f_{\hat{D}})| & \le & \displaystyle\frac{1}{B}\sum_{\beta\in\mathcal{B}}\Big|\frac{1}{n_1}\sum_{j=1}^{n_1}v^{(j)}_\beta - V_\beta(\pi^f_{\hat{D}_\beta})\Big|  \\
     & \le & \displaystyle \sqrt{\frac{1}{2n_1}\log(\frac{2B}{\delta''}}) \\
     & \le & \displaystyle \frac{\epsilon}{8},
\end{array}
\label{eq:thm10.8}
\end{equation}
given that $n_1 = \frac{32\log(2B/\delta'')}{\epsilon^2}$.
Then, (\ref{eq:thm10.4}), (\ref{eq:thm10.5}), (\ref{eq:thm10.6}), (\ref{eq:thm10.7}) and (\ref{eq:thm10.8}) imply, given that $\phi = \frac{\epsilon}{500H^2\sqrt{A}}$,
$$\begin{array}{ccl}
     |\hat{V}^* - \hat{V}(\pi^f_{\hat{D}})| &  \le & \displaystyle 111 H^2\sqrt{A}\phi + \frac{1}{B}\sum_{\beta\in\mathcal{B}}\mathbb{P}(s_1, \pi^f_{\hat{D}_\beta} \to \Bar{L}) + \frac{\epsilon}{8}  \\
     & = & \displaystyle \frac{3\epsilon}{8} + \frac{1}{B}\sum_{\beta\in\mathcal{B}}\mathbb{P}(s_1, \pi^f_{\hat{D}_\beta} \to \Bar{L}).
\end{array}$$ 

On the other hand, if $\pi^f$ does not satisfy $|\hat{V}^* - \hat{V}(\pi^f_{\hat{D}})| \le \epsilon_{\text{demand}}$, we have
$$|\hat{V}^* - \hat{V}(\pi^f_{\hat{D}})| > \epsilon_{\text{demand}} = \frac{\epsilon}{2},$$
so that
$$\displaystyle\frac{1}{B}\sum_{\beta\in\mathcal{B}}\mathbb{P}(s_1, \pi^f_{\hat{D}_\beta} \to \Bar{L}) > \frac{\epsilon}{8}.$$
Then the probability that all $n_2 B$ trajectories do not visit any state in $\Bar{L}$ is 
$$\begin{array}{cl}
     & \prod_{\beta\in\mathcal{B}}\big(1- \mathbb{P}(s_1, \pi^f_{\hat{D}_\beta} \to \Bar{L})\big)^{n_2} \\
     \le & \prod_{\beta\in\mathcal{B}}\exp \big(- n_2 \mathbb{P}(s_1, \pi^f_{\hat{D}_\beta} \to \Bar{L})\big) \\
     = & \exp \big(- n_2 \sum_{\beta\in\mathcal{B}}\mathbb{P}(s_1, \pi^f_{\hat{D}_\beta} \to \Bar{L})\big) \\
     < & \exp (-\epsilon n_2 B / 8).
\end{array}$$
\end{proof}

\begin{lemma}
Suppose all calls to DFS-Learn are successful and $\phi = \frac{\epsilon}{500H^2\sqrt{A}}$. Then

(i) If Learn-on-Simulators terminates with outputting $\pi^f$, then $V^* - V(\pi^f_{\hat{D}}) \le \epsilon$;

(ii) If Learn-on-Simulators selects a meta-policy $\pi^f$ such that $V^* - V(\pi^f_{\hat{D}}) \le \frac{\epsilon}{4}$, then it terminates with outputting $\pi^f$.
\label{lm:67}
\end{lemma}

\begin{proof}[Lemma \ref{lm:67}] The two conclusion are proved separately.

(i) If Learn-on-Simulators terminates with outputting $\pi^f$, we know that $\pi^f$ satisfies 
$$|\hat{V}^* - \hat{V}(\pi^f_{\hat{D}})| \le \epsilon_{\text{demand}} = \epsilon/2$$
and (\ref{eq:thm10.7}) and (\ref{eq:thm10.8}) implies 
$$|{V}(\pi^f_{\hat{D}}) - \hat{V}(\pi^f_{\hat{D}})| \le \phi + \frac{\epsilon}{8}.$$
Then, by Corollary \ref{crl:1}, we have
$$\begin{array}{ccl}
V^* - V(\pi^f_{\hat{D}}) & \le & |V^* - \hat{V}^*| + |\hat{V}^* - \hat{V}(\pi^f_{\hat{D}})| + |{V}(\pi^f_{\hat{D}}) - \hat{V}(\pi^f_{\hat{D}})|  \\
 & \le & 33H\sqrt{A}\phi + \frac{\epsilon}{2} + \phi + \frac{\epsilon}{8} \\
 & \le & \epsilon,
\end{array}$$
given $\phi = \frac{\epsilon}{500H^2\sqrt{A}}$.

\bigskip

(ii) If Learn-on-Simulators selects a meta-policy $\pi^f$ such that $V^* - V(\pi^f) \le \frac{\epsilon}{4}$, then
$$\begin{array}{ccl}
     |\hat{V}^* - \hat{V}(\pi^f_{\hat{D}})| & \le & |V^* - \hat{V}^*| + |V^* - V(\pi^f_{\hat{D}})| + |{V}(\pi^f_{\hat{D}}) - \hat{V}(\pi^f_{\hat{D}})|  \\
     & \le & 33H\sqrt{A}\phi + \frac{\epsilon}{4} + \phi + \frac{\epsilon}{8} \\
     & \le & \frac{\epsilon}{2} = \epsilon_{\text{demand}}.
\end{array}$$
Thus, it terminates with outputting $\pi^f$.
\end{proof}

\noindent
\textbf{Theorem \ref{thm:Learn-on-Simulators}}
\textit{Suppose the call to DFS-Distribution is successful, that is, (\ref{eq:dfs-dist}) and (\ref{eq:dfs-dist-1}) hold, Learn-on-Simulators is invoked with $\mathcal{F}, \hat{V}^*, \epsilon, \delta$. Then with probability at least $1 - \delta$, Learn-on-Simulators terminates with outputting a meta-policy $\hat{\pi}$ such that $V^* - V(\hat{\pi}) \le \epsilon$. Moreover, the number of episodes executed on simulators by Learn-on-Simulators is at most
$$\displaystyle \Tilde{\mathcal{O}}\Big(\frac{H^{11}S^2A^3}{\epsilon^5} \cdot (\log(F))^2 \cdot (\log(\frac{1}{\delta}))^3\Big).$$}

\begin{proof}[Theorem \ref{thm:Learn-on-Simulators}] 
If all calls to DFS-Learn are successful, then with some probability, Learn-on-Simulators terminates with at most $SH$ iterations in the loop, because Lemma \ref{lm:9} guarantees that at least one new state will be added into $L$ in a single iteration, Lemma \ref{lm:8} and Lemma \ref{lm:67} guarantee that after all states are added into $L$ Learn-on-Simulators terminates with outputting an $\epsilon$-optimal meta-policy.

\bigskip

Next, we assign values to $\delta$ in Theorem \ref{thm:dfs-learn}, $\delta'$ in Theorem \ref{thm:boundK} and Corollary \ref{crl:1}, $\delta''$ in Lemma \ref{lm:9}, and compute the sample complexity. 

The probability that a call to DFS-Learn is successful is $1-\delta-\delta'$.
In a single iteration, the probability that at least one call to DFS-Learn fails is at most $Hn_2B(\delta + \delta')$ and the probability of not visiting any state in $\Bar{L}$ is at most $\delta'' + \exp(-\epsilon n_2 B / 8)$; therefore, the probability that Learn-on-Simulators does not output a $\epsilon$-optimal meta-policy is at most $HS\big(Hn_2B(\delta + \delta') + \delta'' + \exp(-\epsilon n_2 B / 8) \big)$.

We let $n_2 = \frac{8\log(3HS/\delta)}{\epsilon B}$ (which can be shown to be smaller than $n_1 = \frac{32\log(2B/\delta'')}{\epsilon^2}$). Moreover, we assign $\delta, \delta' \gets \frac{\epsilon\delta}{48 H^2S\log(3HS/\delta)}, \delta''\gets\frac{\delta}{3HS}$. Then the probability that Learn-on-Simulators does not output a $\epsilon$-optimal meta-policy is at most 
$$\begin{array}{cl}
     & HS\big(Hn_2(\delta + \delta') + \delta'' + \exp(-\epsilon n_2 B / 8) \big) \\
     = & \frac{\delta}{3} + \frac{\delta}{3} + \frac{\delta}{3} = \delta.
\end{array}$$

Each time we call DFS-Learn in Learn-on-Simulators, by Theorem \ref{thm:dfs-learn}, the number of trajectories we collected is at most
$$\mathcal{O}\Big(\frac{HSAB}{\phi^2}\log(\frac{H^3S^2AFB}{\epsilon\delta}\log(\frac{HS}{\delta}))\Big).$$

Besides, in every iteration, we also collect $n_1 B$ trajectories. Since there are at most $HS$ iterations, the sample complexity is at most
$$\begin{array}{cl}
     & \displaystyle HS (n_1 B + Hn_2 B \cdot \mathcal{O}\Big(\frac{HSAB}{\phi^2}\log(\frac{H^3S^2AFB}{\epsilon\delta}\log(\frac{HS}{\delta}))) \Big) \\
     = & \displaystyle \mathcal{O}\Big(\frac{HSB}{\epsilon^2}\log(\frac{HSB}{\delta}) + \frac{H^3S^2AB}{\epsilon\phi^2} \cdot \log(\frac{HS}{\delta}) \cdot \log(\frac{H^3S^2AFB}{\epsilon\delta}\log(\frac{HS}{\delta}))\Big).
\end{array}$$
Recall that $\phi = \frac{\epsilon}{500H^2\sqrt{A}}$ and $B = \frac{2\log(4F/\delta')}{\phi^2}$. Then the sample complexity is equal to 
$$\displaystyle \Tilde{\mathcal{O}}\Big(\frac{H^{11}S^2A^3}{\epsilon^5} \cdot (\log(F))^2 \cdot (\log(\frac{1}{\delta}))^3\Big).$$
\end{proof}

\end{document}